\documentclass[a4paper,11pt]{article}
\pdfoutput=1
\usepackage{a4wide}
\usepackage{amsmath,amsfonts,amssymb,amsthm}
\usepackage[colorlinks,citecolor=blue]{hyperref}
\usepackage{enumitem}
\usepackage{bbm}
\usepackage{natbib}
\usepackage{microtype}
\usepackage{geometry}
\setcitestyle{square}

\newtheorem{theorem}{Theorem}[section]
\newtheorem{lemma}[theorem]{Lemma}
\newtheorem{proposition}[theorem]{Proposition}

\theoremstyle{definition}

\numberwithin{equation}{section}

\def \bN {\mathbb{N}}

\def \bR {\mathbb{R}}

\def \cA {\mathcal{A}}

\def \cN {\mathcal{N}}
\def \cO {\mathcal{O}}

\def \cX {\mathcal{X}}
\def \cY {\mathcal{Y}}

\def \Ba {{\boldsymbol{a}}}

\def \Bx {{\boldsymbol{x}}}

\def \sgn {\,{\rm sgn}\,}

\def \mod {\,{\rm mod}\,}
\def \PWL {\, {\rm PWL}\,}
\def \BIN {\, {\rm BIN}\,}

\def \poly {\, {\rm poly}\,}

\usepackage{booktabs}
\usepackage{array,tabularx}
\newlength{\optimalitytablewidth}

\usepackage{titlesec}
\titleformat{\title}
{\centering\Large\bfseries} 
{} 
{0pt} 
{} 

\begin{document}
	\title{\centering\Large\bfseries Memorization capacity of deep ReLU neural networks\\ characterized by width and depth}
	\author{Xin Yang \thanks{School of Mathematics (Zhuhai), Sun Yat-sen University, Zhuhai, China. E-mail address: \href{mailto:yangx367@mail2.sysu.edu.cn}{yangx367@mail2.sysu.edu.cn}.}
		\and
		Yunfei Yang \thanks{School of Mathematics (Zhuhai) and Guangdong Province Key Laboratory of Computational Science, Sun Yat-sen University, Zhuhai, China. E-mail address: \href{mailto:yangyunfei@mail.sysu.edu.cn}{yangyunfei@mail.sysu.edu.cn}.}
	}
	\date{}
	\maketitle 
	
	\abstract{
		This paper studies the memorization capacity of deep neural networks with ReLU activation. Specifically, we investigate the minimal size of such networks to memorize any \(N\) data points in the unit ball with pairwise separation distance \(\delta\)  and labels from a finite set of size \(C\). Most prior studies characterize the memorization capacity by the number of parameters or neurons. We extend these
		results by constructing ReLU networks with adjustable width \(W\) and depth \(L\) satisfying \(W^2L^2\log W=\cO(N(\log(\delta^{-1})+\log C))\) that can memorize any such dataset.	We also establish width-depth lower bounds that depend explicitly on \(\delta\) and \(C\) for networks capable of memorizing all such datasets. When \(\log(\delta^{-1})\asymp\log N\), our construction achieves \(W^2L^2\asymp N\) for suitable choices of width and depth. This order is optimal if \(\log C\asymp\log N\), and optimal up to logarithmic factors if \(\log C=o(\log N)\).
	}
	
	\section{Introduction}
	
	Deep learning has made remarkable progress in computer vision, natural language processing, and scientific computing \citep{LeCun2015,Krizhevsky2017,Raissi2019}, partly due to the strong approximation capabilities of deep neural networks. Understanding the expressive, approximation, and memorization capabilities of deep neural networks is one of the central themes in machine learning theory, encompassing  classical questions such as the universal approximation property of neural networks \citep{Cybenko1989,Leshno1993}, the width-depth trade-off \citep{Lu2017, Telgarsky2016, Chatziafratis2020} and the generalization and optimization properties in the over-parameterized regime \citep{Zhang2021}. Extensive research has focused on the approximation theory of neural networks across various function spaces \citep{Shen2019,Shen2020,Shen2022,Siegel2023,Yang2025,Yarotsky2017,Yarotsky2018}. In this work, we study the closely related problem of quantifying the memorization capacity of neural networks, which is also called the interpolation problem. It is key to understanding how neural networks capture complex data patterns and to designing parameter-efficient models for resource-constrained settings.
	
	Throughout this paper, \(\log\) denotes the natural	logarithm and \(\log_2\) denotes the base-two logarithm. We consider the minimal size of deep ReLU neural networks required to memorize any set of \(N\) labeled points. Specifically, how large must a ReLU network architecture be so that, for every dataset \(\{(\Bx_i,y_i)\}_{i=1}^N\subseteq\cX\times\cY\)
	with pairwise distinct inputs, there exists a parameter vector
	\(\theta\) such that \(F_\theta(\Bx_i)=y_i\) for all \(i\in[N]\)?
    Existing studies have explored neural networks’ memorization ability under varied activation functions and data assumptions \citep{Yun2019,Bubeck2020,Park2021,Rajput2021,Vershynin2020}. For datasets without structural constraints, early work studied memorization by single-hidden-layer networks with threshold activations \citep{Baum1988} and established sufficient neuron counts for networks with bounded nonlinear
    activations \citep{Huang1998}.
	Moreover, \citet{Zhang2021} demonstrated that a simple two-layer ReLU network with \(2N+d\) parameters can represent any labeling of \(N\) \(d\)-dimensional samples. In contrast, when the input samples satisfy certain separation conditions, recent works showed that the network complexity can be reduced. For instance, \citet{Park2021} proposed a fixed-width deep network achieving memorization of the mild \(\delta\)-separated points with \(N^{2/3}\) parameters; \citet{Rajput2021} demonstrated that threshold networks with
	\(\widetilde{\cO}(d/\delta+N)\) parameters can memorize any such dataset;
	\citet{Vardi2022} further achieved optimal parameter
	efficiency up to logarithmic factors, using
	\(\widetilde{\cO}(\sqrt{N})\) parameters for fixed
	input dimension and polynomially bounded
	\(\delta^{-1}\) and \(C\). Another line of studies in approximation theory of neural networks analyzes the memorization capacity for datasets with uniformly spaced inputs and discrete labels \citep{Shen2022,Siegel2023,Yang2025}. In particular, \citet[Theorem 4.6]{Yang2025} constructed deep ReLU neural networks that memorize
	\(N\) such samples, with width \(W\) and depth \(L\) satisfying \(W^2L^2=\cO(N\log C)\), where \(C\)	is the number of labels.
	
	Despite these advances, much of the existing literature has characterized memorization capacity through parameter or neuron counts \citep{Rajput2021,Vardi2022,Yun2019}. \citet{Vardi2022} achieved nearly optimal
	parameter efficiency with a constant-width
	construction and also considered variable-width
	networks, while \citet{Yang2025} established a more flexible width-depth trade-off for uniformly spaced inputs. However, the width-depth trade-off for memorizing
	arbitrary \(\delta\)-separated inputs has not been fully characterized.
	
	In this work, we consider datasets \(\{(\Bx_i,y_i)\}_{i=1}^N \subseteq \cX \times \cY\), where \(\cX\) is the unit ball of \(\bR^d\) and \(\cY=[C]\), with \(\|\Bx_i-\Bx_j\|\ge \delta \) for any \(i\neq j\). For any \(S,T,\gamma\in\mathbb{N}\) with
	\(1\le \gamma\le S<N\), let	\(G=\lceil S/\gamma\rceil\) and \(R=10N^2\delta^{-1}\sqrt{\pi d}\). Our main result, Theorem~\ref{Theorem 2.1}, shows that	every such dataset can be memorized by a ReLU network with width
	\[
	\begin{aligned}
		W=\max\Biggl\{
		&12\left\lceil
		\sqrt{\frac{N}{SG(T+3)}}
		\right\rceil+5,\\
		&\sum_{\ell=1}^{\gamma}
		\Bigl(
		2^{\left\lceil
			\ell\lceil\log_2(R+1)\rceil/T
			\right\rceil+2}
		+
		2^{\left\lceil
			\ell\lceil\log_2(C+1)\rceil/T
			\right\rceil+2}
		+4
		\Bigr)+2
		\Biggr\}
	\end{aligned}
	\]
	and depth
	\[
	L=3G(T+3)+1.
	\]
	The parameters \(S,T,\gamma\) govern the trade-off between width and depth. 
	For fixed \(d\), we further obtain	memorizing networks satisfying
	\begin{equation}\label{upper bound}
		W^2L^2\log W
		\lesssim
		N\bigl(\log(\delta^{-1})+\log C\bigr).
	\end{equation}
	We use the following asymptotic notation: for two quantities \(A\) and \(B\), \(A \lesssim B\) (or \(A=\cO(B)\) or \(B \gtrsim A\) or \(B= \Omega(A)\)) denotes the statement that \(A\le cB\) for some constant \(c>0\), \(A \asymp B\) means \(A \lesssim B \lesssim A\) or \(A=\Theta(B)\), and we write \(A=o(B)\) if \(A/B\to0\) as \(N\to\infty\). 
	
	Note that, by fixing \(\gamma\) and suitably choosing \(S\) and \(T\), we obtain networks of bounded width.	For fixed \(d\), our constant-width construction recovers the \(\cO(\sqrt{N\log N})\) parameter bound of \citet[Theorem~3.1]{Vardi2022} when	\(\delta^{-1}\) and \(C\) are polynomially bounded in \(N\), and improves their depth upper bound when
	\(\log N=o\left(\max\{\log(\delta^{-1}),\log C\}\right)\).
	Our main technical contribution is to extend their
	construction to achieve a flexible width-depth trade-off.
	Specifically, \(S\) determines the block size for
	partitioning samples and labels.  Each block is processed in \(G=\lceil S/\gamma\rceil\)
	successive groups of at most \(\gamma\) pairs. The pairs within each group are extracted in parallel, with \(T\) controlling the extraction depth. By jointly adjusting these parameters, we control	the width and depth of the resulting network.
	While \citet{Yang2025} established a width-depth trade-off
	for uniformly spaced inputs, our construction applies to arbitrary
	\(\delta\)-separated inputs.

	The VC dimension \citep{Vapnik2015} implies a basic lower bound for memorization since a network capable of memorizing every \(N\)-point dataset must shatter every
	admissible set of \(N\) points. However, this lower bound is relatively weak, as it only exploits the existence of a shattered \(N\)-point set, rather than the ability to shatter every such set. Under the stronger requirement, \citet{Sontag1997} showed that at least \(\Omega(N)\) parameters are necessary. Besides, for ReLU networks, \citet[Proposition~6.6]{Yang2022} proved that memorizing \(N\) distinct points with arbitrary real-valued labels requires at least \(\Omega(N)\) parameters.
    Furthermore, for \(\delta\)-separated inputs, \citet{Siegel2026} established an \(\Omega(N)\) parameter lower bound when \(\delta^{-1}>e^{cN}\) for some constant \(c>0\), even when only two distinct labels are allowed. We complement this result by establishing width-depth lower bounds that depend explicitly on the separation
    distance \(\delta\) and the number of labels \(C\). These bounds imply
    \[
    W^2L^2\gtrsim
    \max\left\{
    \frac{N\log C}{\log N+\log C},
    \min\{N,\log^2(\delta^{-1})\}
    \right\}.
    \]
    We also recover the sharp lower bound
    \(W^2L\gtrsim N\) when \(\log(\delta^{-1})\gtrsim N\). When \(\log(\delta^{-1})\asymp\log N\) and \(\log C\lesssim\log N\),  our adjustable width-depth construction can attain \(W^2L^2\asymp N\). Comparison with the lower bounds shows that this order
    is optimal if \(\log C\asymp\log N\), and optimal up to
    logarithmic factors if \(\log C=o(\log N)\).

	The rest of this paper is organized as follows. Section \ref{Section 2} presents our network construction and the resulting width-depth upper bounds. Section \ref{Section 3} establishes width-depth lower bounds and examines	the optimality of our network constructions. Finally, we summarize the paper in Section \ref{Section 4}. Before we proceed with the formal content of this paper, let us introduce some common notations that will be used in the following sections. We use \(\bN:=\{1,2,\dots\}\) to denote the set of positive integers, and let \(\bN_0:=\bN\cup\{0\}\). For \( n \in \mathbb{N} \), we denote \( [n] := \{1, \ldots, n\} \). Vectors are denoted in bold face such as \(\Bx\in \bR^d\). The Euclidean norm of \(\Bx\) is denoted by \(\|\Bx\|\). We use \(\poly(N)\) to denote a polynomial in \(N\). A quantity is polynomially bounded in \(N\) if it is
	\(\cO(N^a)\) for some fixed \(a\geq0\).
	
	\section{Main result}	\label{Section 2}
	Let us begin by introducing some notations describing the function classes parameterized by deep neural networks \citep{LeCun2015}. Given \( L, N_1, \ldots, N_L \in \bN \), we consider the mapping \( g : \bR^d \to \bR^k \) that can be parameterized by a fully connected ReLU neural network of the following form:
	\[
	g_0(x) = x,
	\]
	\[
	g_{\ell+1}(x) = \sigma(A_\ell g_\ell(x) + b_\ell), \quad \ell = 0, 1, \ldots, L-1,
	\]
	\[
	g(x) = A_L g_L(x) + b_L.
	\]
	Here, \( A_\ell \in \bR^{N_{\ell+1} \times N_\ell} \), \( b_\ell \in \mathbb{R}^{N_{\ell+1}} \) with \( N_0 = d \) and \( N_{L+1} = k \), the activation function \(\sigma(t): =\max\{t,0\}\) is the Rectified Linear Unit function (ReLU) and it is applied component-wisely. We remark that there is no activation function in the output layer, which is the usual convention in applications. The numbers \( W := \max\{N_1, \ldots, N_L\} \) and \( L \) (the number of hidden layers) are called the width and depth  of the neural network, respectively. We denote the set of mappings that can be parameterized by ReLU neural networks with width W and depth L as \(\mathcal{NN}_{d,k}(W,L)\). When the input dimension \(d\) and the output dimension \(k\) are clear from the context, we simplify the notation to \(\mathcal{NN}(W, L)\) for convenience.
	
	 We consider the problem of when the neural network \(\mathcal{NN}(W, L)\) can memorize \(N\) data samples \(\{\Bx_i,y_i\}_{i=1}^N\), where the inputs \(\{\Bx_i\}_{i=1}^N\) are in the \(d\)-dimensional unit ball with pairwise separation distance \(\delta>0\) and the labels  satisfy \(y_i\in[C]\). Our first main result constructs a neural network that can memorize these samples so that it gives an upper bound on the network size required
	for memorization.
	
	\begin{theorem} \label{Theorem 2.1}
		Let \( N, d, C \in \bN \), \(N, C\ge 2\), \( \delta > 0 \) and denote \( R := 10 N^2 \delta^{-1} \sqrt{\pi d} \). Consider a collection of  \( N \) labeled samples \( (\Bx_1, y_1), \dots, (\Bx_N, y_N) \in \bR^d \times \bR \) satisfying the following two conditions:
		\begin{enumerate}
			\renewcommand{\labelenumi}{(\theenumi)}
			\item \(y_i \in [C]\) for all \(i=1,\ldots,N\); 
			\item  \( \|\mathbf{x}_i\| \leq 1 \) for every \( i \) and \( \|\Bx_i - \Bx_j\| \geq \delta \) whenever \( i \neq j \).
		\end{enumerate}
		Then, for any \( S,T,\gamma\in\bN\) with \(1\leq \gamma\leq S<N\), let \(G=\lceil S/\gamma\rceil\).
		There exists a neural network \(F:\bR^d\to\bR\)
		with depth 
		\[
		L=3G(T+3)+1
		\] 
		and width
		\[
		\begin{aligned}
			W=\max\Biggl\{
			&12\left\lceil
			\sqrt{\frac{N}{SG(T+3)}}
			\right\rceil+5,\,\\
			&\sum_{\ell=1}^{\gamma}
			\Bigl(
			2^{\left\lceil
				\ell\lceil\log_2(R+1)\rceil/T
				\right\rceil+2}
			+		2^{\left\lceil
				\ell\lceil\log_2(C+1)\rceil/T
				\right\rceil+2}
			+4
			\Bigr)+2
			\Biggr\}
		\end{aligned}
		\]
		such that \( F(\Bx_i) = y_i \) for every \( i \in [N] \).
	\end{theorem}
	
	The parameters \(S,T,\gamma\) in Theorem~\ref{Theorem 2.1}	provide an explicit width-depth trade-off. For the asymptotic analysis below, fix \(d\). Since the balls \(B(\Bx_i,\delta/2)\) are pairwise disjoint and contained in \(B(\boldsymbol{0},1 +\delta/2)\), comparing their volumes gives \(N\leq(1+2/\delta)^d\). 
	Recalling \(R=10N^2\delta^{-1}\sqrt{\pi d}\), we obtain \(\log R\lesssim \log (\delta^{-1})\). Choose
	\[
	T=\left\lceil
	\max\{\log_2(R+1),\log_2(C+1)\}
	\right\rceil.
	\]
	Fix \(\gamma\in\bN\) independent of \(N\), and set
	\(S=\max\{\gamma,\lfloor\sqrt{N/T}\rfloor\}\). Theorem~\ref{Theorem 2.1} then gives
	\[
	W\asymp 1,\qquad
	L\lesssim\max\{T,\sqrt{NT}\}.
	\]
	For comparison, \citet[Theorem~3.1]{Vardi2022}
	obtained constant-width networks with depth
	\(O(T\sqrt{N/\log N})\). When \(\log N=o(T)\), our bound is asymptotically smaller. If \(d\) is fixed and \(\delta^{-1}\) and \(C\) are polynomially bounded in \(N\), then \(T\asymp\log N\),
	and both depth bounds reduce to \(\cO(\sqrt{N\log N})\). Thus, our construction also recovers their
	\(\cO(\sqrt{N\log N})\)-parameter scaling.

	More generally, to obtain networks with adjustable width and depth,
	we consider parameter choices in the following
	two asymptotic regimes.
	
	\noindent\textbf{Case 1: \(\max\{\log_2(R+1),\log_2(C+1)\}=o(N)\).}
	Choose
	\begin{equation}\label{eq:encoding-depth}
		T=
		\left\lceil
		\max\{\log_2(R+1),\log_2(C+1)\}
		\right\rceil.
	\end{equation}
	For all sufficiently large \(N\), choose any positive
	integer \(\gamma\) satisfying \(\gamma 4^\gamma T\leq N\), and set
	\begin{equation}\label{eq:choice-of-S}
		S=
		\left\lfloor
		\sqrt{\frac{N \gamma}{T4^\gamma}}
		\right\rfloor.
	\end{equation}
	Then \(\gamma\leq S<N\) and
	\(S\asymp\sqrt{N \gamma/(T4^\gamma)}\).
	The sum in the width bound is comparable to \(2^\gamma\).
	Using \(G=\lceil S/\gamma\rceil\asymp S/\gamma\), we obtain
	\begin{equation}\label{eq:balanced-width-depth}
		W\asymp2^\gamma,
		\qquad
		L\asymp\sqrt{\frac{NT}{\gamma4^\gamma}}.
	\end{equation}
	Consequently,
	\begin{equation}\label{eq:balanced-complexity}
		\begin{aligned}
			W^2L^2\log W
			&\asymp
			N\max\{\log_2(R+1),\log_2(C+1)\}\\
			&\lesssim
			N\left( \log(\delta^{-1})+\log C\right).
		\end{aligned}		
	\end{equation}
	These estimates are uniform over all admissible
	choices of \(\gamma\).
	Since \(N/T\to\infty\), \(\gamma\) may grow with \(N\),
	allowing the width and depth to vary while preserving
	\eqref{eq:balanced-complexity}.
	
	\noindent\textbf{Case 2: \(\max\{\log_2(R+1),\log_2(C+1)\}\gtrsim N\).}
	For every admissible \(S,\gamma,T\), the width and depth
	bounds in Theorem~\ref{Theorem 2.1} satisfy
	\[
	WL \gtrsim
	S\max\{\log_2(R+1),\log_2(C+1)\}.
	\]
	Since \(W,S\geq 1\), these bounds give
	\(W^2L\gtrsim\max\{\log_2(R+1),\log_2(C+1)\}\).
	However, interpolation using Lemma~\ref{Lemma 2.4}
	gives memorizing networks with \(W^2L\asymp N\). The bounds in Theorem~\ref{Theorem 2.1} therefore yield no asymptotic improvement in \(W^2L\). Hence, we use Lemma~\ref{Lemma 2.4} to construct a network that memorizes these samples in this case. Besides, since \(L\log W\leq WL \leq W^2L\), we also have
	\[
	W^2L^2\log W \lesssim N^2\lesssim N\left( \log(\delta^{-1})+\log C\right).
    \]

	\subsection{Constructive Proof}\label{Section 2.1}
	
	In this section, we present the core construction and proof of Theorem  \ref{Theorem 2.1}. Building on insights from \citet{Vardi2022} and \citet{Yang2025}, we derive a generalized framework which extends the result in \citet{Vardi2022}, while focusing more directly on the trade-off between width and depth.
	
	Before constructing the network proposed in Theorem \ref{Theorem 2.1}, we state the following proposition that gives useful properties of the neural network class \(\mathcal{NN}(W,L)\) from \citet[Proposition 4.1]{Yang2025}. These  properties will then be used in the construction of our networks later. 
	
	\begin{proposition}\label{Proposition 2.2}
		Let \(f_i\in \mathcal{NN}_{d_i,k_i}(W_i,L_i)\) for \(i\in [n]\). 
		\begin{enumerate}    \renewcommand{\labelenumi}{(\theenumi)}
			\item \textbf{(Inclusiveness)} If \(d_1=d_2\), \(k_1=k_2\) and \(W_1\leq W_2\), \(L_1\leq L_2\), then \(\mathcal{NN}_{d_1,k_1}(W_1\allowbreak,L_1)\subseteq \mathcal{NN}_{d_2,k_2}(W_2,L_2)\);
			\item \textbf{(Composition)} If \(k_1=d_2\), then \(f_2\circ f_1 \in \mathcal{NN}_{d_1,k_2}(\max\{W_1,W_2 \},L_1+L_2)\);
			\item\textbf{(Concatenation)} If \(d_1=d_2\), define \(f(x):=(f_1(x),f_2(x))^\top\), then \(f\in\mathcal{NN}_{d_1,k_1+k_2}(W_1+W_2,\max\{L_1,L_2 \})\).
		\end{enumerate}
	\end{proposition}
	To memorize the \(N\) samples in Theorem \ref{Theorem 2.1}, we construct a network \(F\) formed by the composition of three sub-networks given by \(F=F_3\circ F_2\circ F_1\). We detail the role of each subnetwork as follows. First, the subnetwork \(F_1\) projects the original input points \(\Bx_i\) from \(\bR^d\) to \(\bR\) and we denote the projected points as \(x_i\). Via proper scaling, we can enforce two constraints: (1) \(|x_i-x_j|\geq 2\) for all  \(i\neq j \in [N]\), which ensures the integer parts \(\lfloor x_i\rfloor\) and \(\lfloor x_j\rfloor\) are distinct, and thus have different binary representations; (2) \(0\le x_i\le R\) for all \(i \in [N]\), which ensures the integer part of \(x_i\) can be represented with at most \(\lceil \log_2(R+1) \rceil\) bits. Without loss of generality, we assume that	the samples are indexed so that their projected coordinates satisfy \(x_1<\cdots<x_N\).

	Subsequently, we partition ordered samples \(\{(x_i,y_i)\}_{i=1}^N\), where \(x_i\) denotes the projected input point and \(y_i\) is the corresponding label, into blocks so that each block contains \(S\) such samples. If \(N\) is not a multiple of \(S\), i.e., the last block has fewer than \(S\) samples, we pad this set with supplementary samples \((x_{N+i},y_{N+i})\) where \(x_{N+i}=x_i\) and \(y_{N+i}=0\) for \(i\in [ S\lceil N/S\rceil-N ]\), to ensure the final block also consists of exactly \(S\) samples. These auxiliary entries are used only for encoding. For each \(x_i\), we first represent its integer part \(\lfloor x_i\rfloor\)	as a \(\left\lceil\log_2(R+1)\right\rceil\)-bit binary string by padding leading zeros if its original binary representation has fewer than \(\left\lceil\log_2(R+1)\right\rceil\) bits. Next, we construct \(\lceil N/S\rceil\) integers \(\{u_j\}_{j=1}^{\lceil N/S\rceil}\) to encode the inputs in each block as follows. For the \(j\)-th block, which contains points \(x_{(j-1)S+1},\ldots,x_{jS}\), we concatenate the  \(\left\lceil\log_2(R+1)\right\rceil\)-bit binary strings of all \(\lfloor x_i\rfloor\)'s in this block and then obtain the integer \(u_j\) corresponding to this concatenated binary string by (\ref{integer binary representation}) below. Similarly, we construct integers \(\{w_j\}_{j=1}^{\lceil N/S\rceil}\) using the labels \(\{y_i\}^N_{i=1}\): represent each label \(y_i\) as a \(\left\lceil\log_2(C+1)\right\rceil\)-bit binary string (padding leading zeros if needed), then concatenate these strings for all \(y_i\)'s in the \(j\)-th block to get \(w_j\) by (\ref{integer binary representation}). Then, we construct the subnetwork \(F_2\) which maps the input \(x_i\) to \((x_i, u_j, w_j)^\top\), where \(j=\lceil i/S \rceil\) denotes the block index of \(x_i\). 
	
	Finally, \(F_3\) takes \((x_i,u_j,w_j)^\top\) as input and processes the encoded pairs in block \(j\) in \(G\) successive groups of at most \(\gamma\) pairs.	Within each group, it extracts the input integers and labels in parallel, then compares the extracted integers with \(\lfloor x_i\rfloor\). After all groups have been processed, it returns the
	label \(y_i\) associated with the unique matching integer.
	Each projected sample \((x_i,y_i)\) is associated with the encoding pair \((u_j, w_j)\) of its unique block. Since the projected points are pairwise separated by at least \(2\), their integer parts are distinct. Thus, \(F_3\) can identify the unique matching sample within each block and recover its label. We now present the construction of these sub-networks.
	
	\textbf{Step 1: Construction of \(F_1\)}. We use the following lemma from \citet[Lemma A.2]{Vardi2022} to project the \(d\)-dimensional points \(\Bx_i\) to \(\bR\) and ensure that the norm of the projected points is bounded and that the distance between different points is at least \(2\). Note that \citet{Vardi2022} count the output affine layer in the network depth, whereas we count only hidden layers.
	
	\begin{lemma}\label{Lemma 2.3} Let \(\Bx_1...\Bx_N \in \bR^d \) satisfy \(\|\Bx_i\|\leq 1\) for every \(i \in[N]\) and \(\|\Bx_i-\Bx_j\|\geq\delta\) for every \(i\neq j\). Define \(R:= 10N^2\delta^{-1}\sqrt{\pi d}\). Then, there exists a neural network \( F_1 : \bR^d \to \bR \) with width \(1\) and depth \(1\), such that:
		\begin{enumerate}
			\renewcommand{\labelenumi}{(\theenumi)}
			\item  \( 0\leq F_1(\Bx_i) \leq R\) for every \(i\in[N]\);
			
			\item \(|F_1(\Bx_i)-F_1(\Bx_j)|\geq 2\) for every \(i\neq j\).
		\end{enumerate}
		
	\end{lemma}
	For simplicity of notation, we denote \(F_1(\Bx_i)\) as \(x_i\) for all \(i\in [N]\). The network \(F_1\) reduces the high-dimensional memorization problem to a one-dimensional and encodable form. The projected points are pairwise separated by at least \(2\), ensuring that \(\lfloor x_i \rfloor \neq \lfloor x_j \rfloor\) for all \(i\neq j\). Moreover, \(0\leq x_i\leq R\), so each \(\lfloor x_i\rfloor\) can be represented using \(\lceil\log_2(R+1)\rceil\) bits.
	
	\textbf{Step 2: Construction of \(F_2\)}. To facilitate the description of the subsequent encoding process for data points and labels, we first introduce the basic notation for binary strings and their corresponding numerical representation.
	For an integer \(m\geq1\), a binary string of length \(m\) is an ordered string consisting of \(m\) bits, denoted as \((a_1a_2\ldots a_m)\), where each element \(a_i\) is either \(0\) or \(1\). We denote the corresponding value of this binary string as	
	\begin{equation}\label{integer binary representation}
		\BIN\left(a_1a_2\cdots a_m\right) := \sum_{i=1}^m 2^{m-i}a_i.
	\end{equation}
	For integers \(m \geq 1\), \(n \geq 1\), and two binary strings \((a_1a_2\ldots a_m)\) and \((b_1b_2\ldots b_n)\), define
	\begin{equation}\label{binary representation}
		\BIN\left(a_1a_2\cdots a_m.b_1b_2\cdots b_n\right) := \sum_{i=1}^m 2^{m-i} a_i + \sum_{j=1}^{n} 2^{-j} b_j. 
	\end{equation}
	
	Note that we use \(a_i \in \{0,1\}\) (\(1 \leq i \leq m\)) to denote the \(i\)-th bit of the integer part (with \(a_1 \) as the most significant bit), and \(b_j \in \{0,1\}\) (\(1 \leq j \leq n\)) to denote the \(j\)-th bit of the fractional part (with \(b_1\) as the first bit after the binary point).
	
	Based on the notation for binary strings defined above, we proceed to perform the block-wise encoding operation for the projected points and their corresponding labels. For the one-dimensional projected points \(\{x_i\}_{i=1}^N\) outputted by \(F_1\) with \(x_i \in [0,R]\) and their corresponding  labels \(\{y_i\}_{i=1}^N\) where \(y_i \in [C]\), we first partition the set of projected points and labels into consecutive blocks of size \(S\). If \(N\) is not a multiple of \(S\), then we add redundant samples \((x_{N+i},y_{N+i})\) where \(x_{N+i}=x_i\) and \(y_{N+i}=0\) for \(i\in [ S\lceil N/S\rceil-N ]\) to ensure that the size of the final block is \(S\). Using the block partition and padding described above, we encode each block as follows.
	Represent its \(S\) input integer parts by \(\lceil\log_2(R+1)\rceil\)-bit strings and concatenate	them to obtain \(u_j\) via \eqref{integer binary representation}.	Construct \(w_j\) analogously from the labels using	\(\lceil\log_2(C+1)\rceil\) bits per label.	Leading zeros are added whenever necessary.
	
	
	Next, we construct a subnetwork \(F_2\) which associates each projected point \(x_i\) with the pre-encoded pair \(u_{j_i}\) and \(w_{j_i}\) where \(j_i\) denotes the index of the block containing \(x_i\). Specifically, \(F_2\) maps \(x_i\) to \((x_i, u_{j_i},w_{j_i})^\top\). This design embeds the dataset structure into fixed-length bit segments of network parameters, and provides the necessary discrete information required for the subsequent bit-extraction step. 
	
	Recall that ReLU neural networks can only represent continuous piecewise linear functions, so we first formalize notation for such functions as follows. For \(n \in \bN\), we use \(\PWL(n)\) to denote the set of continuous piecewise linear functions \(g: \bR\to\bR\) with at most \(n\) pieces, that is, there exist at most \( n + 1 \) points \( -\infty = t_0 \leq t_1 \leq \cdots \leq t_n = \infty \) such that \( g \) is linear on the interval \( (t_{i-1}, t_i) \) for all \( i = 1, \ldots, n \). Points at which \( g \) is not differentiable are called breakpoints of \( g \). The following lemma from \citet[Lemma 4.2]{Yang2025} shows that any function in \( \PWL(n) \) can be represented by ReLU networks with specified width and depth.
	
	\begin{lemma}\label{Lemma 2.4} Let \( W, L, n \in \mathbb{N} \) and \( g \in \PWL(n+1) \). Assume that the breakpoints of \( g \)  lie within the bounded interval \( [\alpha, \beta] \). If \( n \leq 6W^2L \), then there exists \( f \in \mathcal{NN}(6W + 2, 2L) \) such that \( f = g \) on \( [\alpha, \beta] \).		
	\end{lemma}
	
	Leveraging this result, we now construct the ReLU network \(F_2\) to realize the desired block-encoded mapping.
	\begin{lemma}\label{Lemma 2.5}
		Let \(0\le x_1<...<x_N\le R\) with \(R>0\). Let \(S \in \bN\) satisfy \(S<N\), \(u_1,...,u_{\lceil N/S \rceil} \in \bN_0\) and \(w_1,...,w_{\lceil N/S \rceil} \in \bN_0\). Then, for any \(W_1,L_1 \in \bN\) such that \(
		3W_1^2 L_1\geq \lceil N/S \rceil\), there exists a neural network \(F_2:\bR \to \bR^3 \) with width \(12W_1+5\) and depth \( 2 L_1\) such that \(F_2(x_i)=(x_i, u_{\lceil i/S\rceil} , w_{\lceil i/S\rceil} )^\top\) for all \(i \in [N]\).	
	\end{lemma}
	
	\begin{proof}
		We define a piecewise linear function \(g_u\) by the following two properties.
		\begin{enumerate}\renewcommand{\labelenumi}{(\theenumi)}
			\item All the possible breakpoints of \(g_u\) are \(x_{tS}\) and \(x_{tS+1}\) for \(t\in [\lceil N/S \rceil -1]\). Furthermore, the function values of \(g_u\) at these points are given by
			\[
			g_u(x_{tS}) = u_t,\quad g_u(x_{tS+1})=u_{t+1},\quad t\in [\lceil N/S \rceil -1].
			\] 
			Then, \(g_u\) is well defined on \([x_S,x_{S\lceil N/S \rceil -S +1}]\) by the piecewise linear property.
			\item \(g_u\) is constant on \((-\infty,x_S]\) and \([x_{S\lceil N/S \rceil -S +1},\infty)\).
		\end{enumerate}
		It is easy to see that the above properties define a piecewise linear function \(g_u\) with at most \(2\lceil N/S \rceil -2\) breakpoints and hence \(g_u\in \PWL(2\lceil N/S \rceil -1)\). By Lemma \ref{Lemma 2.4}, for any \(W_1,L_1 \in \bN\) such that \(3W_1^2 L_1\geq \lceil N/S \rceil\), \(g_u\) can be implemented by a neural network \(F_2^u \in \mathcal{NN}(6W_1 + 2, 2L_1) \) on the interval \([0,R]\). By construction, \(g_u\) is a constant on any interval \([x_{tS+1}, x_{(t+1)S}]\) for \(1\le t\le \lceil N/S\rceil-2\). Together with property (2), this yields \(F_2^u(x_i)=g_u(x_i)=u_{\lceil i/S\rceil}\) for all \(i \in [N]\).
		
		For the integers \(w_1,\ldots ,w_{\lceil N/S \rceil}\), we can similarly define a piecewise linear function \(g_w\) by using \(w_t\) instead of \(u_t\). By Lemma \ref{Lemma 2.4} again, we can get a neural network \(F_2^w \in \mathcal{NN}(6W_1 + 2, 2L_1) \) such that \(F_2^w(x_i)=g_w(x_i)=w_{\lceil i/S\rceil}\) for all \(i \in [N]\).
		
		Finally, we construct \(F_2\) by concatenating \(F_2^u\) and \(F_2^w\) as follows.
		\[F_2(x_i)=\begin{pmatrix}
			\sigma(x_i)\\ F_2^u(x_i)\\	F_2^w(x_i)  	
		\end{pmatrix}=\begin{pmatrix}
			x_i\\	u_{\lceil \frac{i}{S}\rceil}\\w_{\lceil \frac{i}{S}	\rceil} 
		\end{pmatrix}.\] 
		Since each \(x_i\) lies within \([0,R]\), we have \(\sigma(x_i)=x_i\). Thus, by Proposition \ref{Proposition 2.2}, \(F_2\) can be implemented by a network with width \(12W_1+5\) and depth \(2L_1\). 
	\end{proof}

	\textbf{Step 3: Construction of \(F_3\)}. In this step, we construct subnetwork \(F_3\), which leverages a bit-extraction technique. For the input \((x,u,w)^\top\), \(u\)'s binary  string incorporates that of \(\lfloor x\rfloor\), and \(w\)'s binary string contains that of the label \(y\) corresponding to \(x\). The subnetwork \(F_3\) extracts bit segments from \(u\) to identify the integer that matches \(x\) and extracts the corresponding bit segments from \(w\) to retrieve the target label. Thanks to the pairwise spacing of \(\{x_i\}\) previously enforced in the projection step, \(F_3\) can uniquely and correctly recover the label for each \(x_i\).
	
	To formalize the bit-extraction operations employed in \(F_3\), we first define a mapping which extracts a specific range of bits from a binary string and converts them into an integer.
	For \(m\in \bN\) and integers \(1 \leq i \leq j \leq m\), we define a map \(\Gamma_{i:j}^{(m)}:\{0,\cdots,2^m-1\}\to\bN_0\), which extracts the \(i\)-th to \(j\)-th bits of the input, as follows. For \(N\in\{ 0,\cdots,2^m-1\}\), we can write \(N= \BIN (a_1a_2\cdots a_m)= \sum_{i=1}^m 2^{m-i} a_i\) and define
	\[
	\Gamma_{i:j}^{(m)}(N) = \BIN \left(a_ia_{i+1}\cdots a_j\right)=\sum_{k=i}^{j} 2^{j-k} a_k.
	\]
	For example, \(\Gamma_{1:1}^{(1)}(1)=1\), whereas \(\Gamma_{1:1}^{(2)}(1)=0\).
	
	With this bit-extraction mapping defined, we now present the following lemma to complete the construction of \(F_3\), which realizes  bit extraction and matching operations to accurately recover the label for each input. We defer the proof of this lemma to Section \ref{Section 2.2}.
	\begin{lemma}\label{Lemma 2.6}
		
		 Let \(\rho,c,S,T,\gamma\in\bN\) with \(1\leq \gamma\leq S\), and let
		\(G=\lceil S/\gamma\rceil\).
		There exists a neural network \(F_3:\bR^3\to\bR\) with width at most
		\[
		\sum_{\ell=1}^{\gamma}
		\left(
		2^{\lceil\rho\ell/T\rceil+2}
		+2^{\lceil c\ell/T\rceil+2}
		+4
		\right)+2
		\]
		and depth at most \(G(T+3)\) with the following property.
		If an input \((x,u,w)^\top\) satisfies
		\begin{enumerate}
			\renewcommand{\labelenumi}{(\theenumi)}
			\item
			\(u,w\in\bN_0\) have binary representations
			\[
			u=\BIN(a_1a_2\cdots a_{\rho S}),
			\qquad
			w=\BIN(b_1b_2\cdots b_{cS}),
			\]
			and
			\[
			\left|
			\Gamma^{(\rho S)}_{\rho\ell+1:\rho(\ell+1)}(u)
			-
			\Gamma^{(\rho S)}_{\rho k+1:\rho(k+1)}(u)
			\right|\geq 2
			\]
			for all distinct \(\ell,k\in\{0,\ldots,S-1\}\);
			\item
			\(x\geq0\) and
			\[
			\lfloor x\rfloor
			=
			\Gamma^{(\rho S)}_{\rho j+1:\rho(j+1)}(u)
			\]
			for some \(j\in\{0,\ldots,S-1\}\),
		\end{enumerate}
		then
		\[
		F_3\bigl((x,u,w)^\top\bigr)
		=
		\Gamma^{(cS)}_{cj+1:c(j+1)}(w).
		\]

	\end{lemma}	
	In the constant-width construction of \citet{Vardi2022}, the bit-extraction network has fixed width and relies on depth to process the encoded information. 
	Lemma~\ref{Lemma 2.6} refines the bit-extraction mechanism by processing groups successively and extracting and	matching the pairs within each group in parallel. The parameters \(\gamma\) and \(T\) control the group size	and extraction depth, respectively.	Together with the block size \(S\), they determine the width-depth trade-off in Theorem~\ref{Theorem 2.1}.

	\textbf{Proof of Theorem \ref{Theorem 2.1}}. We now combine the previous three steps and give a proof of Theorem \ref{Theorem 2.1}. 
	\begin{proof}
		We begin by invoking Lemma \ref{Lemma 2.3} to construct a network \(F_1:\bR^d \to \bR\), for which the following properties hold: \(0\le F_1(\Bx_i)\leq R:= 10N^2\delta^{-1}\sqrt{\pi d}\) for all \(i \in [N]\),  and  \(|F_1(\Bx_i)-F_1(\Bx_j)|\geq2\) for all \(i\neq j\). For notational convenience, we denote the output of \(F_1\) on \(\Bx_i\) as \(x_i\) (\(i=1,\ldots,N\)). Without loss of generality, we assume \(\{x_i\}_{i=1}^N\) is in increasing order, otherwise, we reorder the indices. Note that the width of \(F_1\) is \(1\) and its depth is \(1\).
		
		Next, we construct integers
		\(u_1,\ldots,u_{\lceil N/S\rceil}\) and
		\(w_1,\ldots,w_{\lceil N/S\rceil}\) to encode the
		integer parts of the projected points \(\{x_i\}_{i=1}^N\)
		and their associated labels \(\{y_i\}_{i=1}^N\) in a block-wise manner. Each \(u_i\)'s binary string comprises \(S\cdot\left\lceil\log_2(R+1)\right\rceil\) bits while each \(w_i\) corresponds to a binary string of \(S\cdot \left\lceil\log_2(C+1)\right\rceil\). Specifically, we detail the constructive procedure of \(\{u_j\}_{j=1}^{\lceil N/S \rceil}\) as follows.
		\begin{enumerate}     \renewcommand{\labelenumi}{(\theenumi)}
			\item We first represent each \(\lfloor x_i \rfloor\) in terms of \(\left\lceil\log_2(R+1)\right\rceil\)-bit binary notation. If its binary representation has fewer than \(\left\lceil\log_2(R+1)\right\rceil\) bits, we pad leading zeros (i.e., on the most significant bit side) to meet the length requirement. 
			\begin{equation}\label{represent x_i}
				\lfloor x_i\rfloor=\BIN(a_{i,1} a_{i,2}\cdots a_{i,\left\lceil\log_2(R+1)\right\rceil})=\sum_{k=1}^{\left\lceil\log_2(R+1)\right\rceil}2^{\left\lceil\log_2(R+1)\right\rceil-k}a_{i,k}.
			\end{equation}
			For notational convenience, we denote the binary string of \(\lfloor x_i \rfloor\) as  \(s_i\), where  
			\begin{equation}\label{s_i}
				s_i:=a_{i,1}a_{i,2}\cdots a_{i,\left\lceil\log_2(R+1)\right\rceil}. 
			\end{equation}
			
			\item  We then partition \(\{x_i\}_{i=1}^N\) into blocks of size \(S\), and use all the points in the \(j\)-th block to construct \(u_j\), i.e., 
			\[
			u_j=u_j\left(\lfloor x_{(j-1)S+1}\rfloor,\lfloor x_{(j-1)S+2}\rfloor,\ldots,\lfloor x_{jS}\rfloor \right), \quad 1\leq j\leq \left\lceil N/S\right\rceil.
			\]
			If \(N\) is not a multiple of \(S\), we pad the last block with supplementary points to ensure that each block contains exactly \(S\) points. More precisely, we let \(x_{N+i}= x_i\) and \(s_{N+i}=s_i\) for \(i\in [S\left\lceil N/S\right\rceil - N]\). This choice of \(x_{N+i}\) is to ensure that the points in the last block are separated so that we can apply Lemma \ref{Lemma 2.6} below. To construct \(u_j\), we combine all the binary strings of \(\lfloor x_i \rfloor\) in the \(j\)-th block together:	
			\begin{equation}\label{u_j}
				u_j=\BIN\bigl(s_{(j-1)S+1}s_{(j-1)S+2}\ldots s_{jS}\bigr).
			\end{equation}
		\end{enumerate}
		
		Similarly, we use the labels \(\{y_i\}_{i=1}^N\) to construct \(\{w_j\}_{j=1}^{\lceil N/S\rceil}\). \begin{enumerate}     \renewcommand{\labelenumi}{(\theenumi)}
			\item We represent each \(y_i\) as a \(\left\lceil\log_2(C+1)\right\rceil \)-bit binary string by padding leading zeros to its binary representation if necessary.
			\begin{equation}\label{represent y_i}
				y_i=\BIN(b_{i,1}b_{i,2}\cdots b_{i,\left\lceil\log_2(C+1)\right\rceil})=\sum_{k=1}^{\left\lceil\log_2(C+1)\right\rceil}2^{\left\lceil\log_2(C+1)\right\rceil-k}b_{i,k}.
			\end{equation}
			For notational convenience, we denote the binary string of \(y_i \) as \(t_i\), where 
			\begin{equation}\label{t_i}
				t_i:=b_{i,1}b_{i,2}\cdots b_{i,\left\lceil\log_2(C+1)\right\rceil}. 
			\end{equation}
			
			\item We partition \(\{y_i\}_{i=1}^N\) into blocks of size S, and use labels in the \(j\)-th block to construct \(w_j\). If \(N\) is not a multiple of \(S\), for \( i\in [S\lceil N/S\rceil-N]\), set
			\(y_{N+i}=0\) and let \(t_{N+i}\) be the binary string
			consisting of \(\lceil\log_2(C+1)\rceil\) zeros, to ensure that each block contains exactly \(S\) labels. Specifically, we concatenate the binary strings \(t_i\) of \(y_i\) in the \(j\)-th block to construct \(w_j\) as
			\begin{equation}\label{w_j}
				w_j=\BIN\bigl(t_{(j-1)S+1}t_{(j-1)S+2}\ldots t_{jS}\bigr).
			\end{equation}
		\end{enumerate}
		By construction, for each projected point \(x_i\) and its corresponding label \(y_i\), let \( j_i := \lceil i/S \rceil\) be its block index and 
	    \(k:=(i-1)\bmod S\),  so that \(k+1\) is its position within the block. Then, we can recover \(\lfloor x_i\rfloor\) and \(y_i\) from \(u_{j_i}\) and \(w_{j_i}\) as follows:
		\begin{equation}\label{x_i corresponds to u_{j_i}}
			\lfloor x_i \rfloor=\Gamma^{(S\left\lceil\log_2(R+1)\right\rceil)}_{k \cdot \left\lceil\log_2(R+1)\right\rceil + 1: (k+1) \cdot \left\lceil\log_2(R+1)\right\rceil}(u_{j_i})=\BIN(s_i),
		\end{equation}
		\begin{equation}\label{y_i corresponds to w_{j_i}}
			y_i= \Gamma^{(S\left\lceil\log_2(C+1)\right\rceil)}_{k \cdot \left\lceil\log_2(C+1)\right\rceil + 1: (k+1) \cdot \left\lceil\log_2(C+1)\right\rceil}(w_{j_i})=\BIN(t_i).
		\end{equation}
		
		Since \(u_j,w_j\in\bN_0\) for all \(j\in[\lceil N/S\rceil]\), \(0\le x_i\le R\) for any \(i\in [N]\), the conditions of Lemma \ref{Lemma 2.5} are satisfied for any \(W_1,L_1\in\bN\) such that \(3W_1^2L_1\geq\lceil N/S\rceil\). Thus, we can invoke Lemma \ref{Lemma 2.5} to construct network \(F_2\) such that 
		\begin{equation}\label{subnetwork F_2}
			F_2(x_i)=\begin{pmatrix}
				x_i\\u_{j_i}\\	w_{j_i} 
			\end{pmatrix}.
		\end{equation}
		The width of \(F_2\) is \( 12W_1+5\) and its depth is \(2L_1\).  
		
		Next, by the construction of \(\{u_j\}_{j=1}^{\lceil N/S \rceil}\) and \(\{w_j\}_{j=1}^{\lceil N/S \rceil}  \), we have 
		\[
		u_j=\BIN \big(s_{(j-1)S+1}s_{(j-1)S+2}\ldots s_{jS}\big),
		\] 
		where \(s_{(j-1)S+q}=a_{(j-1)S+q,1}\ldots a_{(j-1)S+q,\left\lceil\log_2(R+1)\right\rceil}\), and 
		\[
		w_j=\BIN\bigl(t_{(j-1)S+1}t_{(j-1)S+2}\ldots t_{jS}\bigr),
		\] 
		where \(t_{(j-1)S+q}=b_{(j-1)S+q,1}\ldots b_{(j-1)S+q,\left\lceil\log_2(C+1)\right\rceil}\), for each \(j \in [\lceil N/S \rceil], q\in [S]\). By construction, for any \(l,m\in \{0,\ldots S-1\}\) with \(l\neq m\), we have 
		\[
		\left| \Gamma^{(S\left\lceil\log_2(R+1)\right\rceil)}_{m \cdot \left\lceil\log_2(R+1)\right\rceil + 1: (m+1) \cdot \left\lceil\log_2(R+1)\right\rceil}(u_j)-\Gamma^{(S\left\lceil\log_2(R+1)\right\rceil)}_{l \cdot \left\lceil\log_2(R+1)\right\rceil + 1: (l+1) \cdot \left\lceil\log_2(R+1)\right\rceil}(u_j) \right|\geq 2,
		\] 
		which satisfies the input condition \((1)\) of Lemma \ref{Lemma 2.6}. Furthermore, by equation \ref{x_i corresponds to u_{j_i}}, for \( k =(i-1)(\mod S)\), so that \(k+1\) is the position of \(x_i\) within its block \(j_i=\lceil i/S\rceil\), we have \(x_i=F_1(\Bx_i)\ge 0\) and \( \lfloor x_i \rfloor=\Gamma^{(S\left\lceil\log_2(R+1)\right\rceil)}_{k \cdot \left\lceil\log_2(R+1)\right\rceil + 1: (k+1) \cdot \left\lceil\log_2(R+1)\right\rceil}(u_{j_i})\), satisfying the input condition \((2)\) of Lemma \ref{Lemma 2.6}. Thus, the output of \(F_2\) satisfies all input requirements of Lemma \ref{Lemma 2.6}, allowing us to directly invoke this lemma to construct the extractor subnetwork \(F_3\).  For the input \((x_i,u_{j_i},w_{j_i})^\top\) and the unique index \(k\), we obtain
		\begin{equation}\label{subnetwork F_3}
			F_3\left(
			\begin{pmatrix}
				x_i\\
				u_{j_i}\\
				w_{j_i}
			\end{pmatrix}
			\right)
			=
			\Gamma^{
				\left(S\left\lceil\log_2(C+1)\right\rceil\right)
			}_{
				k\left\lceil\log_2(C+1)\right\rceil+1:
				(k+1)\left\lceil\log_2(C+1)\right\rceil
			}(w_{j_i})
			=
			y_i.
		\end{equation}
		As guaranteed by Lemma \ref{Lemma 2.6}, the network \(F_3\) has width
		\[
		\sum_{\ell=1}^{\gamma}
		\left(
		2^{\lceil \ell \lceil\log_2(R+1)\rceil/T\rceil+2}
		+2^{\lceil \ell \lceil \log_2(C+1)\rceil/T\rceil+2}
		+4
		\right)+2
		\]
		and depth \(G(T+3)\). Finally, we can define the composite network \(F:=F_3\circ F_2 \circ F_1: \bR^d\rightarrow \bR\) such that \(F(\Bx_i)=y_i\) for any \(i\in[N]\).
		
		Having constructed the network \(F\), we can determine its width \(W\) and depth \(L\) by Proposition \ref{Proposition 2.2}: 
		\begin{equation}\label{width and depth of F}
			\begin{gathered}
			W=\max\left\{12W_1+5,  
			\sum_{\ell=1}^{\gamma}
			\left(
			2^{\lceil \ell\lceil\log_2(R+1)\rceil/T\rceil+2}
			+2^{\lceil \ell\lceil\log_2(C+1)\rceil/T\rceil+2}
			+4
			\right)+2,
			\right\}\\
			L=1+2L_1+G(T+3),
	    	\end{gathered}
		\end{equation}
		where \(W_1\) and \(L_1\) are respectively the width and depth parameters of the subnetwork \(F_2\), \(S\) is the sample block size and \(T\) is a parameter of \(F_3\).
		
		Now, we choose \(L_1=G(T+3)\), and recall that Lemma \ref{Lemma 2.5} imposes the constraint that \(3W_1^2L_1\geq \lceil N/S \rceil\). From this constraint, we may choose \(W_1=\lceil \sqrt N/\sqrt{ SG(T+3)} \rceil\).  Substituting this to (\ref{width and depth of F}) yields the following expression for the width of \(F\): 
		\[
		\begin{aligned}
			W=\max\Biggl\{
			&12\left\lceil
			\sqrt{\frac{N}{SG(T+3)}}
			\right\rceil+5,\,\\
			&\sum_{\ell=1}^{\gamma}
			\Bigl(
			2^{\left\lceil
				\ell\lceil\log_2(R+1)\rceil/T
				\right\rceil+2}
			+		2^{\left\lceil
				\ell\lceil\log_2(C+1)\rceil/T
				\right\rceil+2}
			+4
			\Bigr)+2
			\Biggr\}.
		\end{aligned}
		\]
		Finally, with the choice \(L_1=G(T+3)\) in (\ref{width and depth of F}), the depth of \(F\) is \(L=3G(T+3)+1\), which completes the proof.
	\end{proof}
	
	\subsection{Proof of Lemma \ref{Lemma 2.6}}\label{Section 2.2}
	To prove Lemma \ref{Lemma 2.6}, let us first introduce two key auxiliary lemmas. The following lemma from \citet[Lemma 4.3]{Yang2025}, which is a slight modification
	of \citet[Lemma 13]{Bartlett2019}, provides the essential technique we use to realize the bit extraction. The value of a binary string is defined in
	\eqref{binary representation}.
	\begin{lemma}\label{Lemma 2.7}
		Let \(m,n,L\in\bN\) with \(m\leq n\). There exists \( f_{n,m,L} \in \mathcal{NN}_{1,2}(2^{\lceil m/L \rceil+ 2} + 2, L )\) such that, for any \(b=\BIN(0.b_1\cdots b_n)\) with \(b_i\in\{0,1\}\),
		\[
		f_{n,m,L}(b) = (\BIN (b_1 \cdots b_m.0), \BIN(0.b_{m+1} \cdots b_n))^\top.
		\]
		If \(m=n\), the empty remaining binary string is understood to be zero.
	\end{lemma}
	We also need a lemma to verify the matching relationship between the extracted data segments and the input. The following lemma from \citet[Lemma A.8]{Vardi2022} constructs a simple network to check whether the integer part of the projected input falls within a specific interval.
	\begin{lemma}\label{Lemma 2.8}
		There exists a network \(F : \bR^2 \to \bR\) with width \(2\) and depth \(2\) such that 
		\[
		F\left( \begin{pmatrix} x \\ y \end{pmatrix} \right) = \begin{cases}
			1 & x \in [y, y+1], \\
			0 & x > y + \frac{3}{2} \text{ or } x < y - \frac{1}{2}.
		\end{cases}
		\]
	\end{lemma}
	We now prove Lemma~\ref{Lemma 2.6} using Lemmas~\ref{Lemma 2.7} and~\ref{Lemma 2.8}.

	\begin{proof}
		First, we scale \(u\) and \(w\) to the interval \([0,1)\)
		to satisfy the input requirement of Lemma \ref{Lemma 2.7}.
		Since \(u\) and \(w\) are encoded by \(\rho S\) and \(cS\)
		binary digits, respectively, we have
		\(0\leq u<2^{\rho S}\), \(0\leq w<2^{cS}\), and
		\[
		\frac{u}{2^{\rho S}}
		=
		\BIN(0.a_1\cdots a_{\rho S}),
		\qquad
		\frac{w}{2^{cS}}
		=
		\BIN(0.b_1\cdots b_{cS}).
		\]
		We allocate a channel, initialized to zero, to store the accumulated
		label and rescale \(u\) and \(w\) as follows,
		\[
		\begin{pmatrix}
			x\\
			u\\
			w
		\end{pmatrix}
		\longmapsto
		\begin{pmatrix}
			x\\
			\dfrac{u}{2^{\rho S}}\\
			\dfrac{w}{2^{cS}}\\
			0
		\end{pmatrix}.
		\]
	
		For \(i=0,\ldots,S-1\), define
		\[
		\begin{aligned}
			s_i
			&:=
			\Gamma^{(\rho S)}_{\rho i+1:\rho(i+1)}(u)
			=
			\BIN(a_{\rho i+1}\cdots a_{\rho(i+1)}),\\
			t_i
			&:=
			\Gamma^{(cS)}_{ci+1:c(i+1)}(w)
			=
			\BIN(b_{ci+1}\cdots b_{c(i+1)}).
		\end{aligned}
		\]
		By assumption, the integers \(s_i\) are pairwise separated
		by at least \(2\), so the intervals \(\{[s_j,s_j+1)\}_{j=0}^{S-1}\) are disjoint. Thus there is a unique \(j\in\{0,\ldots,S-1\}\) such that \(\lfloor x\rfloor=s_j\). To recover \(t_j\), we partition the pairs
		\(\{(s_i,t_i)\}_{i=0}^{S-1}\) into \(G\) consecutive
		groups of at most \(\gamma\) pairs and process them
		successively, extracting and matching the pairs in
		each group in parallel.

		We first process \((s_0,t_0),\ldots,(s_{\gamma-1},t_{\gamma-1})\).
		For every \(\ell\in[\gamma]\), Lemma \ref{Lemma 2.7} gives networks
		\[
		\begin{aligned}
			f^u_{\rho S,\rho\ell,T}
			&\in
			\mathcal{NN}_{1,2}
			\bigl(2^{\lceil\rho\ell/T\rceil+2}+2,T\bigr),\\
			f^w_{cS,c\ell,T}
			&\in
			\mathcal{NN}_{1,2}
			\bigl(2^{\lceil c\ell/T\rceil+2}+2,T\bigr),
		\end{aligned}
		\]
		satisfying
		\[
		f^u_{\rho S,\rho\ell,T}
		\left(\frac{u}{2^{\rho S}}\right)
		=
		\begin{pmatrix}
			\BIN(a_1\cdots a_{\rho\ell})\\
			\BIN(0.a_{\rho\ell+1}\cdots a_{\rho S})
		\end{pmatrix}
		\]
		and
		\[
		f^w_{cS,c\ell,T}
		\left(\frac{w}{2^{cS}}\right)
		=
		\begin{pmatrix}
			\BIN(b_1\cdots b_{c\ell})\\
			\BIN(0.b_{c\ell+1}\cdots b_{cS})
		\end{pmatrix}.
		\]
		
		Apply these \(2\gamma\) networks in parallel.
		For \(\ell=1\), the first components are \(s_0\) and \(t_0\).
		For \(\ell=2,\ldots,\gamma\), compute \(s_{\ell-1}\) and
		\(t_{\ell-1}\) by
		\[
		\begin{aligned}
			s_{\ell-1}
			&=
			\BIN(a_1\cdots a_{\rho\ell})
			-
			2^\rho\BIN(a_1\cdots a_{\rho(\ell-1)}),\\
			t_{\ell-1}
			&=
			\BIN(b_1\cdots b_{c\ell})
			-
			2^c\BIN(b_1\cdots b_{c(\ell-1)}).
		\end{aligned}
		\]
		Retain the second outputs of
		\(f^u_{\rho S,\rho \gamma,T}\) and \(f^w_{cS,c\gamma,T}\),
		\[
		\BIN(0.a_{\rho \gamma+1}\cdots a_{\rho S}),
		\qquad
		\BIN(0.b_{c\gamma+1}\cdots b_{cS}),
		\]
		for the subsequent groups.
		The affine output layer computes these linear combinations
		and retains the two remaining fractions.
		Preserve \(x\) and the accumulated label value through all \(T\)
		hidden layers. Since both are nonnegative, this requires two additional
		neurons per layer. Thus, on inputs satisfying the assumptions, the extraction
		step realizes
		\begin{equation}\label{eq:parallel-extraction}
		\begin{aligned}
			&\begin{pmatrix}
				x\\		
				\dfrac{u}{2^{\rho S}}\\
				\dfrac{w}{2^{cS}}\\
				0
			\end{pmatrix}
			\longmapsto
			\begin{pmatrix}
				x\\
				s_0\\
				\vdots\\
				s_{\gamma-1}\\
				\BIN(0.a_{\rho \gamma+1}\cdots a_{\rho S})\\
				t_0\\
				\vdots\\
				t_{\gamma-1}\\
				\BIN(0.b_{c\gamma+1}\cdots b_{cS})\\
				0
			\end{pmatrix}\\
			\in \mathcal{NN}&\left(\sum_{\ell=1}^{\gamma}
			\left(
			2^{\lceil\rho\ell/T\rceil+2}
			+2^{\lceil c\ell/T\rceil+2}
			+4
			\right)+2, T\right).
		\end{aligned}	
		\end{equation}		
		
		Next, we check whether each extracted integer equals
		\(\lfloor x\rfloor\).
		By Lemma \ref{Lemma 2.8}, there exists a network
		\(H:\bR^2\to\bR\) with width \(2\) and depth \(2\) such that
		\[
		H\bigl((x,s_i)^\top\bigr)
		=
		\begin{cases}
			1,
			& x\in[s_i,s_i+1],\\
			0,
			& x<s_i-\dfrac12
			\text{ or }x>s_i+\dfrac32.
		\end{cases}
		\]
		For \(i=0,\ldots,S-1\), let
		\(\tilde{y}_i:=H((x,s_i)^\top)\).
		If \(i=j\), then \(s_i\leq x<s_i+1\), so
		\(\tilde{y}_i=1\).
		If \(i\neq j\), the separation assumption implies either
		\(s_j\geq s_i+2\) or \(s_j\leq s_i-2\).
		In the first case, \(x\geq s_i+2\); in the second,
		\(x<s_i-1\).
		Hence \(\tilde{y}_i=0\) whenever \(i\neq j\).
		Apply \(\gamma\) copies of \(H\) in parallel to \((x,s_0)^\top,\ldots,(x,s_{\gamma-1})^\top\), while preserving \(x\), the extracted labels \(t_0,\ldots,t_{\gamma-1}\), the two remaining binary fractions, and the accumulated label value. This gives
		\begin{equation}\label{eq:parallel-matching}
			\begin{pmatrix}
				x\\
				s_0\\
				\vdots\\
				s_{\gamma-1}\\
				\BIN(0.a_{\rho \gamma+1}\cdots a_{\rho S})\\
				t_0\\
				\vdots\\
				t_{\gamma-1}\\
				\BIN(0.b_{c\gamma+1}\cdots b_{cS})\\
				0
			\end{pmatrix}
			\longmapsto
			\begin{pmatrix}
				x\\
				\tilde{y}_0\\
				\vdots\\
				\tilde{y}_{\gamma-1}\\
				\BIN(0.a_{\rho \gamma+1}\cdots a_{\rho S})\\
				t_0\\
				\vdots\\
				t_{\gamma-1}\\
				\BIN(0.b_{c\gamma+1}\cdots b_{cS})\\
				0
			\end{pmatrix}
			\in\mathcal{NN}\left(3\gamma+4,2\right).
		\end{equation}		
		
		To output \(t_i\) when \(\tilde{y}_i=1\) and zero when
		\(\tilde{y}_i=0\), we define the one-hidden-layer network
		\[
		\phi:
		\begin{pmatrix}
			\tilde{y}_i\\
			t_i
		\end{pmatrix}
		\longmapsto
		\sigma\bigl(2^{c+1}\tilde{y}_i-2^{c+1}+t_i\bigr).
		\]
		Since \(0\leq t_i\leq2^c-1\), we have
		\[
		\phi\bigl((\tilde{y}_i,t_i)^\top\bigr)
		=
		\begin{cases}
			t_i, & i=j,\\
			0, & i\neq j.
		\end{cases}
		\]
		Apply \(\phi\) in parallel to
		\((\tilde{y}_0,t_0)^\top,\ldots,
		(\tilde{y}_{\gamma-1},t_{\gamma-1})^\top\).
		Preserve \(x\), the two remaining binary fractions,
		and the channel storing the accumulated label value.
		Add the sum of the \(\gamma\) outputs of \(\phi\) to the value
		stored in this channel, which is initially zero.
		This gives 
		\begin{equation}\label{eq:label-accumulation}
				\begin{pmatrix}
					x\\
					\tilde{y}_0\\
					\vdots\\
					\tilde{y}_{\gamma-1}\\
					\BIN(0.a_{\rho \gamma+1}\cdots a_{\rho S})\\
					t_0\\
					\vdots\\
					t_{\gamma-1}\\
					\BIN(0.b_{c\gamma+1}\cdots b_{cS})\\
					0
				\end{pmatrix}
				\longmapsto
				\begin{pmatrix}
					x\\
					\BIN(0.a_{\rho \gamma+1}\cdots a_{\rho S})\\
					\BIN(0.b_{c\gamma+1}\cdots b_{cS})\\
					\displaystyle
					0+\sum_{i=0}^{\gamma-1}
					\phi\bigl((\tilde{y}_i,t_i)^\top\bigr)
				\end{pmatrix}
				\in\mathcal{NN}\left(\gamma+4,1\right).
		\end{equation}
		The width bound in \eqref{eq:parallel-extraction}
		dominates those in
		\eqref{eq:parallel-matching} and
		\eqref{eq:label-accumulation}.
		By Proposition \ref{Proposition 2.2}, 
		\[
		\begin{aligned}
		&\begin{pmatrix}
			x\\
			\dfrac{u}{2^{\rho S}}\\
			\dfrac{w}{2^{cS}}\\
			0
		\end{pmatrix}
		\longmapsto
		\begin{pmatrix}
			x\\
			\BIN(0.a_{\rho \gamma+1}\cdots a_{\rho S})\\
			\BIN(0.b_{c\gamma+1}\cdots b_{cS})\\
			\displaystyle
			0+\sum_{i=0}^{\gamma-1}
			\phi\bigl((\tilde{y}_i,t_i)^\top\bigr)
		\end{pmatrix}\\
		\in  \mathcal{NN}& \left(\sum_{\ell=1}^{\gamma}
		\left(
		2^{\lceil\rho\ell/T\rceil+2}
		+2^{\lceil c\ell/T\rceil+2}
		+4
		\right)+2, T+3\right).
		\end{aligned}
		\]

		If \(G\geq2\), process the remaining groups successively
		using the same extraction, matching, and label accumulation
		steps as in the first group. For \(k=2,\ldots,G\), the input to the \(k\)-th group
		is the output after processing the first \(k-1\) groups:
		\[
		\begin{pmatrix}
			x\\
			\BIN(0.a_{\rho(k-1)\gamma+1}\cdots a_{\rho S})\\
			\BIN(0.b_{c(k-1)\gamma+1}\cdots b_{cS})\\
			\displaystyle
			\sum_{i=0}^{(k-1)\gamma-1}
			\phi\bigl((\tilde{y}_i,t_i)^\top\bigr)
		\end{pmatrix}.
		\]
		The two binary fractions contain \(\rho(S-(k-1)\gamma)\) and \(c(S-(k-1)\gamma)\) digits, respectively. By Lemma~\ref{Lemma 2.7}, for each
		\(1\le \ell\le \min\{\gamma,S-(k-1)\gamma\}\),
		apply \(f^u_{\rho(S-(k-1)\gamma),\rho\ell,T}\)
		to \(\BIN(0.a_{\rho(k-1)\gamma+1}\cdots a_{\rho S})\)
		and \(f^w_{c(S-(k-1)\gamma),c\ell,T}\)
		to \(\BIN(0.b_{c(k-1)\gamma+1}\cdots b_{cS})\)
		in parallel. Their integer output components are
		\(\BIN(a_{\rho(k-1)\gamma+1}
		\cdots a_{\rho((k-1)\gamma+\ell)})\)
		and
		\(\BIN(b_{c(k-1)\gamma+1}
		\cdots b_{c((k-1)\gamma+\ell)})\).
		For \(\ell=1\), they are \((s_{(k-1)\gamma},t_{(k-1)\gamma})\).
		For \(\ell\geq2\), the affine output layer computes
		\[
		\begin{aligned}
			s_{(k-1)\gamma+\ell-1}
			&=
			\BIN(a_{\rho(k-1)\gamma+1}\cdots
			a_{\rho((k-1)\gamma+\ell)})\\
			&\quad
			-2^\rho
			\BIN(a_{\rho(k-1)\gamma+1}\cdots
			a_{\rho((k-1)\gamma+\ell-1)}),\\
			t_{(k-1)\gamma+\ell-1}
			&=
			\BIN(b_{c(k-1)\gamma+1}\cdots
			b_{c((k-1)\gamma+\ell)})\\
			&\quad
			-2^c
			\BIN(b_{c(k-1)\gamma+1}\cdots
			b_{c((k-1)\gamma+\ell-1)}),
		\end{aligned}
		\]
		and retains
		\[
		\BIN(0.a_{\rho\min\{k\gamma,S\}+1}\cdots a_{\rho S}),
		\qquad
		\BIN(0.b_{c\min\{k\gamma,S\}+1}\cdots b_{cS})
		\]
		as the remaining binary fractions.
		
		Apply the matching and label accumulation steps in
		\eqref{eq:parallel-matching} and
		\eqref{eq:label-accumulation} to the extracted pairs.
		The value stored in the label accumulation channel
		is updated by
		\[
		\begin{aligned}
			&\sum_{i=0}^{(k-1)\gamma-1}
			\phi\bigl((\tilde{y}_i,t_i)^\top\bigr)
			+
			\sum_{i=(k-1)\gamma}^{\min\{k\gamma,S\}-1}
			\phi\bigl((\tilde{y}_i,t_i)^\top\bigr)\\
			&\qquad=
			\sum_{i=0}^{\min\{k\gamma,S\}-1}
			\phi\bigl((\tilde{y}_i,t_i)^\top\bigr).
		\end{aligned}
		\]
		This realizes
		\begin{equation}\label{eq:group-iteration}
			\begin{pmatrix}
				x\\
				\BIN(0.a_{\rho(k-1)\gamma+1}\cdots a_{\rho S})\\
				\BIN(0.b_{c(k-1)\gamma+1}\cdots b_{cS})\\
				\displaystyle
				\sum_{i=0}^{(k-1)\gamma-1}
				\phi\bigl((\tilde{y}_i,t_i)^\top\bigr)
			\end{pmatrix}
			\longmapsto
			\begin{pmatrix}
				x\\
				\BIN(0.a_{\rho\min\{k\gamma,S\}+1}\cdots a_{\rho S})\\
				\BIN(0.b_{c\min\{k\gamma,S\}+1}\cdots b_{cS})\\
				\displaystyle
				\sum_{i=0}^{\min\{k\gamma,S\}-1}
				\phi\bigl((\tilde{y}_i,t_i)^\top\bigr)
			\end{pmatrix}.
		\end{equation}
		
		By Proposition \ref{Proposition 2.2}, each group can be
		processed by a network with width at most
		\[
		\sum_{\ell=1}^{\gamma}
		\left(
		2^{\lceil\rho\ell/T\rceil+2}
		+
		2^{\lceil c\ell/T\rceil+2}
		+4
		\right)+2
		\]
		and depth at most \(T+3\).
		Consequently, processing all \(G\) groups requires the
		same width bound and depth at most \(G(T+3)\).
		
		Since \(G=\lceil S/\gamma\rceil\), the first \(G-1\) groups
		process \(\gamma\) pairs each, and the last group processes
		the remaining \(S-(G-1)\gamma\) pairs. Hence, after all \(G\)
		groups have been processed, the label accumulation channel
		contains
		\[
		\sum_{i=0}^{S-1}
		\phi\bigl((\tilde{y}_i,t_i)^\top\bigr).
		\]
		There is a unique \(j\in\{0,\ldots,S-1\}\) such that
		\(\lfloor x\rfloor=s_j\). Therefore,
		\(\tilde{y}_j=1\) and \(\tilde{y}_i=0\) for every
		\(i\neq j\). It follows that
		\[
			\sum_{i=0}^{S-1}
			\phi\bigl((\tilde{y}_i,t_i)^\top\bigr)
			=
			\phi\bigl((\tilde{y}_j,t_j)^\top\bigr)
			=t_j
			=\Gamma^{(cS)}_{cj+1:c(j+1)}(w).
		\]
		Finally, an affine output map selects the value stored in
		the label accumulation channel. Hence,
		\[
		F_3\bigl((x,u,w)^\top\bigr)
		=
		\Gamma^{(cS)}_{cj+1:c(j+1)}(w).
		\]
		The initial scaling and the final affine output map add no
		hidden layers. Therefore, the resulting network \(F_3\)
		has width at most
		\[
		\sum_{\ell=1}^{\gamma}
		\left(
		2^{\lceil\rho\ell/T\rceil+2}
		+
		2^{\lceil c\ell/T\rceil+2}
		+4
		\right)+2
		\]
		and depth at most \(G(T+3)\). This completes the proof.
	\end{proof}

	\section{Optimality}\label{Section 3}
	In Theorem \ref{Theorem 2.1}, we consider \(N\) labeled samples \((\Bx_1,y_1),\ldots,(\Bx_N,y_N) \in\bR^d\times \bR \) satisfying two key conditions: 
	\begin{enumerate}
		\renewcommand{\labelenumi}{(\theenumi)}
		\item \(y_i \in [C]\) for all \(i\in [N]\).
		\item  \( \|\mathbf{x}_i\| \leq 1 \) for every \( i \in [N]\) and \( \|\Bx_i - \Bx_j\| \geq \delta \) whenever \( i \neq j \).
	\end{enumerate}
	For fixed \(d\), we have shown that these samples can be memorized by a ReLU network whose width \(W\) and depth \(L\) satisfy
	\begin{equation}\label{upper bound in Section 3}
		 W^2L^2\log W	\lesssim N\bigl(\log(\delta^{-1})+\log C\bigr).
	\end{equation}
	In this section, we investigate whether this width-depth bound is optimal for memorizing every admissible dataset.

	We first examine the role of condition (1), which restricts the
	labels to \([C]\). Suppose instead that the labels may take
	arbitrary values in a bounded interval \([a,b]\), where \(a<b\).
	One can show that any fully connected ReLU network capable of
	memorizing all such labels must satisfy \(W^2L\gtrsim N\) when
	\(L\geq2\) \citep[Proposition 6.6]{Yang2022}.
	To see this, fix \(N\) distinct inputs \(\Bx_1,\ldots,\Bx_N\), and
	let \(f_\theta\) be a ReLU network with \(P\) parameters
	\(\theta\in\bR^P\). Define
	\[
	F(\theta)
	=
	\bigl(f_\theta(\Bx_1),\ldots,f_\theta(\Bx_N)\bigr).
	\]
	If the network can memorize every label vector in \([a,b]^N\), then
	\([a,b]^N\subseteq F(\bR^P)\).
	Since \(F\) is a continuous piecewise polynomial function of
	\(\theta\), it is locally Lipschitz. It follows that the Hausdorff
	dimension of \(F(\bR^P)\) is at most \(P\)
	\citep[Theorem 2.8]{Evans2015}. Since \([a,b]^N\) has Hausdorff
	dimension \(N\), we obtain \(N\leq P\). For a fully connected
	network with fixed input and output dimensions, width \(W\), and
	depth \(L\geq2\), we have \(P\lesssim W^2L\), and hence \(W^2L\gtrsim N\). Additionally, recall from Lemma \ref{Lemma 2.4} that a network with width \(6W+2\) and depth \(2L\) satisfying \(6W^2L\geq N\) suffices to memorize arbitrary real values for any set of \(N\) distinct points. Therefore, \(W^2L\lesssim N\) is sufficient for memorizing these \(N\) points.  Together with the lower bound, this shows that when the labels \(\{y_i\}_{i=1}^N\) take arbitrary values in an interval, the optimal network complexity for exact memorization is \(W^2L\asymp N\), independently of the separation distance \(\delta\).
	
	In contrast, for discrete labels \(y_i\in[C]\), \eqref{upper bound in Section 3} gives \(W^2L^2\log W\lesssim N\log N\) when \(C\) is fixed and \(\delta^{-1}\) is polynomially	bounded in \(N\). With an admissible choice satisfying
	\(\log W\asymp\log N\), this yields \(W^2L^2\lesssim N\), whereas exact memorization of arbitrary real-valued labels requires \(W^2L\gtrsim N\). This difference relies on the finite-label assumption. In our construction, each label is encoded by a fixed-length binary string and recovered through bit extraction, which does not directly extend to arbitrary real-valued labels.
	
	The finite-label assumption is also used in our lower-bound	proof.
	For any fixed admissible set of \(N\) points, the network class must realize all
	\(C^N\) possible label assignments. Our proof exploits these finitely
	many assignments through a counting argument to derive the width-depth lower bounds. For continuous labels, however, the network must realize every label vector in the \(N\)-dimensional set \([a,b]^N\), rather than finitely many assignments. This yields the stronger bound \(W^2L\gtrsim N\) by the dimension argument above.

	We now return to discrete labels \(y_i\in[C]\) and	examine the role of the separation condition (2). We establish the lower bounds for \(d=1\).	Embedding each input \(x_i\) as \((x_i,0,\ldots,0)\) shows that these bounds also apply in every dimension \(d\ge1\).
	Indeed, restricting a network to this coordinate axis
	does not increase its width or depth.

	In this one-dimensional setting the structural properties of the sample points directly affect the network width \(W\) and depth \(L\) required for memorization. For comparison, the special case of fixed, uniformly spaced points, such as	\(x_i=i\), has been studied by \citet{Siegel2023,Shen2022,Yang2025}. In particular, \citet[Theorem 4.6]{Yang2025} showed that uniform data samples can be memorized by a network whose width \(W\) and depth \(L\) satisfy \(W^2L^2\lesssim N \log C\). For binary labels, the VC-dimension bound of \citet{Bartlett2019} gives the lower bound \(W^2L^2\log (WL)\gtrsim N\). Thus, the upper and lower bounds match up to logarithmic factors in this setting. However, these bounds do not account for the requirement	to memorize every admissible set of \(\delta\)-separated
	inputs.

	For general \(\delta\)-separated inputs, \citet{Siegel2026} proved that memorizing every admissible \(N\)-point dataset in the unit ball requires \(\Omega(N)\) parameters when \(\delta < e^{-cN}\) for some constant \(c\). This implies \(W^2L\gtrsim N\), and this lower bound is sharp since \(\cO(N)\) parameters are always sufficient. Our lower bounds recover this sharp order when \(\log(\delta^{-1})\gtrsim N\) and extend the analysis
	to \(\log(\delta^{-1})=o(N)\), with explicit dependence	on both \(\delta\) and \(C\). We also use these lower bounds to establish the optimality	of the width-depth trade-off in Theorem~\ref{Theorem 2.1} in certain regimes.
	
	To prove these bounds, we use the following version of Warren's lemma \citep[Lemma 17]{Bartlett2019} to bound
	the number of sign patterns of polynomial inequalities.
	Throughout this section, we use the convention that
	\(\sgn(x)=-1\) for \(x<0\) and \(\sgn(x)=1\) for \(x\ge 0\).

	\begin{lemma}\label{Lemma 3.1}
		Let \(P, M, D \in \bN \) with \(P\leq M \) and let \(f_1,\ldots, f_M\) be polynomials of degree at most \(D\) with \(P\) variables. Define \(K\) as the number of distinct sign vectors attained by these polynomials, i.e., 
		\[
		K:=\left|\left\{\left(\sgn(f_1(\Ba)),\ldots, \sgn(f_M(\Ba))\right): \Ba \in\bR^P\right\}\right|.
		\]  
		Then we have \(K\leq 2(2eMD/P)^P\).	
	\end{lemma}
	
	Building on Warren's lemma, we establish lower bounds on the width and depth of ReLU networks required for memorizing any \(N\) points satisfying the norm and separation conditions.
	
	\begin{theorem}\label{Theorem 3.2}
	Let \(N,C,W,L\in\bN\) with \(N\geq3\), \(C\geq2\),
	\(W\geq2\), and \(L\geq1\), and let \(\delta>0\) be such that
	\([-1,1]\) contains \(N\) pairwise	\(\delta\)-separated points. Suppose that for every collection of \(N\) labeled samples \(\{(x_i,y_i)\}_{i=1}^N\subset\mathbb R\times\mathbb R\)
	satisfying
	\begin{enumerate}
		\renewcommand{\labelenumi}{(\theenumi)}		
		\item  \( |x_i| \leq 1\) for all \( i \in [N]\) and \( |x_i - x_j| \geq \delta \) for all \( i \neq j \),
		\item  \(y_i\in[C]\),
	\end{enumerate}
	there exists \(f\in\mathcal{NN}(W,L)\) such that \(f(x_i)=y_i\) for all \(i\in[N]\).
	Then the following lower bounds hold:
	\begin{equation}\label{eq:main-lower-bounds}
			W^2L
			\gtrsim
			\min\left\{
			N,\,\log(\delta^{-1})+\sqrt{N\log C}
			\right\}.
	\end{equation}
	and
	\begin{equation}\label{eq:quadratic-lower-bound}
		 W^2L^2 \gtrsim
		\max\left\{
		\frac{N\log C}{\log N+\log C},
		\min\{N,\log^2(\delta^{-1})\}
		\right\}.
	\end{equation}
	The implied constants are independent of
	\(N,C,W,L,\delta\).
	\end{theorem}

	\begin{proof}
		Let
		\[
		T=\lfloor 2/\delta\rfloor+1,
		\qquad
		x_i=-1+(i-1)\delta,\quad i\in[T].
		\]
		Then \(x_1,\ldots,x_T\in[-1,1]\) are pairwise \(\delta\)-separated. Since \(N\) pairwise \(\delta\)-separated points exist in \([-1,1]\), we have \((N-1)\delta\le 2\), and hence \(T\ge N\). Thus, every \(N\)-element subset of \(\{x_1,\ldots,x_T\}\), with arbitrary labels in \([C]\), can be memorized by \(\cN\cN(W,L)\).
		Since every network in \(\mathcal{NN}(W,L)\) can be	embedded into a rectangular architecture with exactly \(L\) hidden layers of width \(W\), we may consider this architecture, which has \(P:= (L-1)W^2 + (L+2)W + 1\le 2W^2L\) parameters.  Let \(f_\Ba \in \mathcal{NN}(W,L)\) be the neural network function parameterized by \(\Ba \in \bR^P\). For each \(\Ba\in\bR^P\) and the fixed inputs \(x_1,\ldots,x_T\), define
		\[
		s(\Ba)=
		\left(
		\sgn\left(f_{\Ba}(x_i)-r-\frac12\right)
		\right)_{\substack{i\in[T]\\r\in[C-1]}}
		\in\{-1,1\}^{T(C-1)}
		\]
		\[
		K=\left|\{s(\Ba):\Ba\in\bR^P\}\right|.
		\]
		The vector \(s(\Ba)\) records the comparisons of
		\(f_{\Ba}(x_i)\) with the thresholds \(r+1/2\),	and \(K\) is the number of distinct vectors obtained as \(\Ba\) varies over \(\bR^P\). We call each such vector a sign pattern. At each input, these comparisons distinguish all labels	in \([C]\).

		\smallskip
		\noindent\textbf{(i) A lower bound on \(K\).}
		On the fixed inputs \(x_1,\ldots,x_N\), all \(C^N\)
		label assignments can be memorized. Since distinct
		assignments give distinct sign patterns,
		\begin{equation}\label{eq:label-pattern-count}
			K\ge C^N.
		\end{equation}
		Let \(m=\lfloor N/2\rfloor\).
		By \citet[Theorem~1]{Serra2018}, every \(f_{\Ba}\)
		has at most \((W+1)^L\) affine pieces.
		Memorizing alternating labels \(1,2,1,2,\ldots\)
		on \(x_1,\ldots,x_N\) requires at least \(N-1\)
		affine pieces. Thus,
		\begin{equation}\label{eq:affine-piece-bound}
			(W+1)^L\ge N-1\ge m.
		\end{equation}
		
		Suppose first that \(T\ge2N\).
		Consider the disjoint pairs
		\[
		(x_{2j-1},x_{2j}),\qquad
		1\le j\le\left\lfloor\frac{T-1}{2}\right\rfloor.
		\]
		Select any \(m\) pairs and include \(x_T\) if \(N\)
		is odd. For each selected pair \((x_{2j-1},x_{2j})\), assign
		a label from \(\{1,\ldots,\lfloor C/2\rfloor\}\)
		to \(x_{2j-1}\) and a label from
		\(\{\lfloor C/2\rfloor+1,\ldots,C\}\) to \(x_{2j}\).
		If \(N\) is odd, assign any label in \([C]\) to \(x_T\).
		Let \(\mathcal D\) be the resulting family of labeled
		data sets. Each data set has \(N\) pairwise
		\(\delta\)-separated inputs, and
		\begin{equation}\label{eq:paired-dataset-count}
			|\mathcal D|
			=
			\binom{\lfloor(T-1)/2\rfloor}{m}
			\left(
			\left\lfloor\frac C2\right\rfloor
			\left\lceil\frac C2\right\rceil
			\right)^m C^{N-2m}.
		\end{equation}
		
		Fix a sign pattern \(s\), and choose \(\Ba\) such
		that \(s(\Ba)=s\). Set
		\(\tau=\lfloor C/2\rfloor+1/2\) and define
		\[
		J_s=
		\left\{
		1\le j\le\left\lfloor\frac{T-1}{2}\right\rfloor:
		f_{\Ba}(x_{2j-1})<\tau\le f_{\Ba}(x_{2j})
		\right\}.
		\]
		\(J_s\) depends only on \(s\). For each \(j\in J_s\), continuity gives a point
		\(t_j\in(x_{2j-1},x_{2j}]\) with \(f_{\Ba}(t_j)=\tau\).	If \(j<k\), then \(x_{2k-1}\) lies strictly between \(t_j\) and \(t_k\), and \(f_{\Ba}(x_{2k-1})<\tau\). Hence, \(t_j\) and \(t_k\) cannot belong to the same affine piece. It follows that
		\[
		|J_s|\le(W+1)^L.
		\]
		
		Every data set in \(\mathcal D\) memorized by a	network with sign pattern \(s\) must select its \(m\) pairs from \(J_s\). Once these pairs are fixed, all \(N\) points are fixed, and \(s\) determines at most one assignment of labels to these inputs.
		Thus, all networks with sign pattern \(s\)	together memorize at most
		\[
		\binom{|J_s|}{m}\le\binom{(W+1)^L}{m}
		\]
		data sets in \(\mathcal D\).
		Since every data set in \(\mathcal D\) can be
		memorized,
		\[
		K\binom{(W+1)^L}{m}\ge|\mathcal D|.
		\]
		Since \(T\ge 2N\), we have
		\[
		\left\lfloor\frac{T-1}{2}\right\rfloor-m+1
		\ge\frac T4,
		\qquad
		\left\lfloor\frac C2\right\rfloor
		\left\lceil\frac C2\right\rceil
		\ge\frac{C^2}{5}.
		\]
		Hence, \eqref{eq:paired-dataset-count} yields
		\[
		K\ge C^N
		\left(\frac{T}{20(W+1)^L}\right)^m.
		\]
		Together with \eqref{eq:label-pattern-count},
		this yields
		\begin{equation}\label{eq:grid-pattern-lower-bound}
			\log K\ge N\log C+
			m\left[\log\frac{T}{20(W+1)^L}\right]_+,
		\end{equation}
		where \([t]_+:=\max\{t,0\}\).
		
		If \(T<2N\), then \eqref{eq:affine-piece-bound}
		and \(N\ge3\) give \(T<2N\le3(W+1)^L\), so the
		positive-part term in \eqref{eq:grid-pattern-lower-bound}
		vanishes. Thus, \eqref{eq:label-pattern-count} also establishes
		\eqref{eq:grid-pattern-lower-bound} in this case.
		Finally, using \(T\ge2\delta^{-1}\) and \(m\ge N/3\)
		in \eqref{eq:grid-pattern-lower-bound}, we obtain
		\begin{equation}\label{eq:pattern-lower-bound}
			\log K\gtrsim
			N\left(
			\log C+
			\left[\log(\delta^{-1})-L\log(W+1)\right]_+
			\right).
		\end{equation}
		Here, the fixed logarithmic shift is absorbed into
		the implied constant using \([t]_+\le[t-\log10]_++\log10\) and	\(\log C\ge\log2\).

		We next derive two upper bounds on \(K\) using Lemma~\ref{Lemma 3.1}.  We may assume \(P\le N\).	Indeed, if \(P>N\), then \(P\le2W^2L\) implies
		\(W^2L>N/2\). Since \(L\ge1\), we also have \(W^2L^2\ge W^2L>N/2\).	These estimates imply all the stated bounds. 
		Since \(P\ge 3W+1\ge7\), the assumption \(P\le N\)
		implies \(N\ge7\). Hence,
		\(\delta\le2/(N-1)\le1/3\), and
		\(\log(4eT)\lesssim\log(\delta^{-1})\).
		
		\smallskip
		\noindent\textbf{(ii) First upper bound on \(K\).}
		We apply Lemma~\ref{Lemma 3.1} successively to the
		hidden layers and the output layer. Let \(P_\ell\) be the number of parameters up to the \(\ell\)-th hidden
		layer. Then
		\[
		P_\ell=2W+(\ell-1)W(W+1),\qquad \ell\in[L].
		\]
		Set \(P_{L+1}=P\). We have
		\[
		W\ell\le P_\ell\le 2W^2\ell,\qquad \ell\in[L+1],
		\]
		and hence
		\[
		\sum_{\ell=1}^{L+1}P_\ell
		\le W^2(L+1)(L+2)
		\le 6W^2L^2.
		\]
		Since ReLU is piecewise linear, the functions
		\[
		\Ba\longmapsto f_{\Ba}(x_i)-r-\frac12,
		\qquad i\in[T],\quad r\in[C-1],
		\]
		may not be polynomials on \(\bR^P\), whereas
		Lemma~\ref{Lemma 3.1} applies to polynomials.
		Following the layerwise partition argument of
		\citet[Theorem~7]{Bartlett2019}, we construct
		partitions \(\cA_0,\ldots,\cA_L\)
		of \(\bR^P\) with the following properties.
		
		\begin{enumerate}
			\item[(1)]
			\(\cA_0=\{\bR^P\}\).
			For each \(\ell\in[L]\), \(\cA_\ell\)
			is obtained by subdividing the sets in
			\(\cA_{\ell-1}\) according to the signs
			of the layer-\(\ell\) preactivations at
			\(x_1,\ldots,x_T\), and
			\[
			|\cA_\ell|
			\le
			|\cA_{\ell-1}|\,
			2\left(\frac{2eTW\ell}{P_\ell}\right)^{P_\ell}.
			\]
			
			\item[(2)]
			For each \(\ell\in[L+1]\),
			\(A\in\cA_{\ell-1}\), and \(i\in[T]\),
			the net input to each neuron in layer \(\ell\)
			at \(x_i\), as a function of \(\Ba\in A\),
			is a fixed polynomial in the first \(P_\ell\)
			parameters of total degree at most \(\ell\).
			Here, layer \(L+1\) is the affine output layer.
		\end{enumerate}
		By (2) with \(\ell=L+1\), the \(T(C-1)\) expressions
		\(f_{\Ba}(x_i)-r-1/2\) are polynomials in the \(P\)
		parameters of degree at most \(L+1\) on each
		\(A\in\cA_L\).	Applying Lemma~\ref{Lemma 3.1} to these \(T(C-1)\) polynomials,
		with \(P\) variables and degree at most \(L+1\), is valid since \(P\le T(C-1)\), and yields
		\[
		\left|\{s(\Ba):\Ba\in A\}\right|
		\le
		2\left(\frac{2eT(C-1)(L+1)}{P}\right)^P.
		\]
		Summing the above bound over \(A\in\cA_L\)
		and using the estimate for \(|\cA_L|\)
		obtained by iterating (1), we obtain
		\[
		\begin{aligned}
			K
			&\le
			\sum_{A\in\cA_L}
			\left|\{s(\Ba):\Ba\in A\}\right|\\
			&\le
			2^{L+1}
			\prod_{\ell=1}^{L}
			\left(\frac{2eTW\ell}{P_\ell}\right)^{P_\ell}
			\left(\frac{2eT(C-1)(L+1)}{P}\right)^P\\
			&\le
			C^P(4eT)^{\sum_{\ell=1}^{L+1}P_\ell},
		\end{aligned}
		\]
		where the last inequality uses
		\(P_{L+1}=P\), \(P_\ell\ge W\ell\), and
		\(L+1\le\sum_{\ell=1}^{L+1}P_\ell\).
		Consequently,
		\[
		\log K
		\le P\log C+
		\left(\sum_{\ell=1}^{L+1}P_\ell\right)\log(4eT)
		\le
		2W^2L\log C+6W^2L^2\log(4eT).
		\]
		It follows that
		\begin{equation}\label{eq:pattern-upper-bound-1}
			\log K
			\lesssim
			W^2L^2\log(\delta^{-1})+W^2L\log C.
		\end{equation}
		
		\smallskip
		\noindent\textbf{(iii) Second upper bound on \(K\).}
		We use the counting argument of	\citet[Eq.~(2.19)]{Siegel2026}. For each input and each of the \(2^{WL}\) choices of
		hidden-neuron activation states, this argument counts
		\(WL\) hidden preactivation polynomials and one output
		polynomial. In our setting, subtracting the thresholds
		\(r+\tfrac12\), \(r\in[C-1]\), from the output polynomial
		gives \(C-1\) polynomials instead. The preactivation signs determine the actual activation	states successively through the layers.
		The corresponding output comparison signs then
		determine \(s(\Ba)\). Thus, \(s(\Ba)\) is determined by the signs of at most
		\[
		T2^{WL}(WL+C-1)
		\]
		polynomials in \(\Ba\), each of degree at most \(L+1\).
		Since \(P\le T\), Lemma~\ref{Lemma 3.1} gives
		\[
		K\le
		2\left(
		\frac{2e(L+1)T2^{WL}(WL+C-1)}{P}
		\right)^P.
		\]
		Taking logarithms and using \(W(L+1)\leq P\leq2W^2L\) and
		\(T\le 3\delta^{-1}\), we obtain
		\begin{equation}\label{eq:pattern-upper-bound-2}
		\log K
		\lesssim
		W^2L\bigl(\log(\delta^{-1})+WL+\log C\bigr).
		\end{equation}

		\smallskip
		\noindent\textbf{(iv) Derivation of the lower bounds.}
		Combining \eqref{eq:pattern-lower-bound} with
		\eqref{eq:pattern-upper-bound-1} and
		\eqref{eq:pattern-upper-bound-2}, respectively, gives
		\begin{equation}\label{eq:lower-bound-layerwise}
			\begin{aligned}
				&W^2L^2\log(\delta^{-1})+W^2L\log C\\
				&\qquad\gtrsim
				N\left(
				\log C+
				[\log(\delta^{-1})-L\log(W+1)]_+
				\right),
			\end{aligned}
		\end{equation}
		and
		\begin{equation}\label{eq:lower-bound-activation}
			\begin{aligned}
				&W^2L\bigl(\log(\delta^{-1})+WL+\log C\bigr)\\
				&\qquad\gtrsim
				N\left(
				\log C+
				[\log(\delta^{-1})-L\log(W+1)]_+
				\right).
			\end{aligned}
		\end{equation}

		Since \(\log(W+1)\le W\), we have  \(\log(\delta^{-1})
		\le WL+[\log(\delta^{-1})-L\log(W+1)]_+ \).
		If \(WL\le\log C+[\log(\delta^{-1})-L\log(W+1)]_+\),
		then
		\[
		\log(\delta^{-1})+WL+\log C\le
			3\left(
			\log C+[\log(\delta^{-1})-L\log(W+1)]_+
			\right).
		\]
		Hence,  \eqref{eq:lower-bound-activation} implies  \(W^2L\gtrsim N\).
		Otherwise, \(\log(\delta^{-1})+WL+\log C<3WL\), so
		\eqref{eq:lower-bound-activation} and \(WL\le W^2L\) imply \((W^2L)^2 \ge W^2L\cdot WL\gtrsim N\log C\).
		Moreover, \(\log(\delta^{-1})<2WL\le 2W^2L\).
		Consequently,
		\[
		W^2L\gtrsim
		\log(\delta^{-1})+\sqrt{N\log C}.
		\]
		Thus, the two cases yield
		\[
			W^2L
			\gtrsim
			\min\left\{
			N,\,\log(\delta^{-1})+\sqrt{N\log C}
			\right\},
		\]
		which proves \eqref{eq:main-lower-bounds}.

		For \(W^2L^2\), the inequality \(W^2L\le W^2L^2\) and the nonnegativity of the positive-part term in \eqref{eq:lower-bound-layerwise} imply
		\begin{equation}\label{eq:label-grid-bound}
			W^2L^2\gtrsim
			\frac{N\log C}{\log(\delta^{-1})+\log C}.
		\end{equation}
		
		We also apply the counting argument in (ii) to the
		fixed inputs \(x_1,\ldots,x_N\).
		Since \(P\le N\) and all \(C^N\) label assignments
		on these inputs are realizable,
		\[
		C^N\le C^P(4eN)^{\sum_{\ell=1}^{L+1}P_\ell}.
		\]
		Taking logarithms and using \(P\le2W^2L\),
		\(\sum_{\ell=1}^{L+1}P_\ell\le6W^2L^2\), and \(L\ge1\),
		we obtain
		\begin{equation}\label{eq:label-quadratic-bound}
			W^2L^2\gtrsim
			\frac{N\log C}{\log N+\log C}.
		\end{equation}
		Since \(N\le T\le 3\delta^{-1}\) and \(C\ge 2\), we have
		\(\log N+\log C \lesssim \log(\delta^{-1})+\log C\).
		Thus, \eqref{eq:label-quadratic-bound} also implies
		\eqref{eq:label-grid-bound}.
	
		If \(\log(\delta^{-1})\ge2L\log(W+1)\), then
		\[
		[\log(\delta^{-1})-L\log(W+1)]_+
		\ge\frac12\log(\delta^{-1}).
		\]
		Equation \eqref{eq:lower-bound-layerwise} consequently
		yields
		\[
		W^2L^2\bigl(\log(\delta^{-1})+\log C\bigr)
		\gtrsim N\bigl(\log(\delta^{-1})+\log C\bigr),
		\]
		and thus \(W^2L^2\gtrsim N\).
		Otherwise,
		\(\log(\delta^{-1})<2L\log(W+1)\le2WL\), giving
		\(W^2L^2>{(\log^2(\delta^{-1}))}/4\).
		Therefore,
		\begin{equation}\label{eq:separation-quadratic-bound}
			W^2L^2\gtrsim
			\min\{N,\log^2(\delta^{-1})\}.
		\end{equation}
		Combining \eqref{eq:label-quadratic-bound} and
		\eqref{eq:separation-quadratic-bound} proves \eqref{eq:quadratic-lower-bound}
		and completes the proof.

	\end{proof}

	We now compare the upper and lower bounds for fixed \(d\),
	under the assumptions of Theorem~\ref{Theorem 3.2}.
	If
	\(\max\{\log(\delta^{-1}),\log C\}\gtrsim N\),
	then \eqref{eq:main-lower-bounds} gives \(W^2L\gtrsim N\), matching the upper bound
	\(W^2L\lesssim N\) established in Case~2 of the discussion
	following Theorem~\ref{Theorem 2.1}.
	Thus, the memorizing networks constructed using	Lemma~\ref{Lemma 2.4} attain the optimal order \(W^2L\asymp N\). Moreover, when \(\log(\delta^{-1})\gtrsim\sqrt{N}\) or
	\(\log C\gtrsim\log N\), \eqref{eq:quadratic-lower-bound} gives \(W^2L^2\gtrsim N\).
	Since Lemma~\ref{Lemma 2.4} yields \(W^2L^2\lesssim N\) with constant depth, the optimal order \(W^2L^2\asymp N\) is attained in this regime.

	We next examine the optimality of the width-depth trade-off in Theorem~\ref{Theorem 2.1}. Suppose that \(\max\{\log(\delta^{-1}),\log C\}=o(N)\).	When \(\log(\delta^{-1})\asymp\log N\) and
	\(\log C\lesssim\log N\), \eqref{upper bound in Section 3} and \eqref{eq:quadratic-lower-bound}  give
	\[
	W^2L^2\log W\lesssim N\log N,
	\qquad
	W^2L^2\gtrsim \frac{N\log C}{\log N}.
	\]
	Thus, the gap between the upper and lower bounds is at most
	\(\cO((\log N)^2/(\log W\log C))\). For the construction in Theorem~\ref{Theorem 2.1}, the choice \eqref{eq:encoding-depth} gives \(T\asymp\log N\), so the condition \(\gamma4^\gamma T\le N\) admits \(\gamma\asymp\log N\) for which \eqref{eq:balanced-width-depth} gives \(\log W\asymp\log N\). With these choices,
	the construction is optimal when \(\log C\asymp\log N\), and optimal up to	a factor of \(\cO(\log N/\log C)\) when \(\log C=o(\log N)\).

	More generally, suppose that
	\(\max\{\log(\delta^{-1}),\log C\}\lesssim(\log N)^k\)
	for some fixed \(k\ge1\). Then  \eqref{upper bound in Section 3} yields
	\[
	W^2L^2\log W\lesssim N(\log N)^k.
	\]
	Since \eqref{eq:quadratic-lower-bound} implies \(W^2L^2\gtrsim N/\log N\),
	the gap is at most
	\(\cO((\log N)^{k+1}/\log W)\).
	Choosing \(\gamma\asymp\log N\) in \eqref{eq:balanced-width-depth} gives
	\(\log W\asymp\log N\), so the constructed networks
	are optimal up to a factor of \(\cO((\log N)^k)\).
	
	The following table summarizes the optimality
	conclusions for fixed \(d\), under the assumptions
	of Theorem~\ref{Theorem 3.2}.
	\setlength{\optimalitytablewidth}{\linewidth}
	\begin{table}[htbp]
		\centering
		\begin{minipage}{\optimalitytablewidth}
			\caption{Optimality of the memorizing network constructions.}
			\label{tab:optimality}
			\footnotesize
			\setlength{\tabcolsep}{3pt}
			\renewcommand{\arraystretch}{1.35}
			\renewcommand{\tabularxcolumn}[1]{m{#1}}
			
			\begin{tabularx}{\linewidth}{@{}
					>{\raggedright\arraybackslash}m{0.28\linewidth}
					c
					c
					>{\centering\arraybackslash}X@{}}
				\toprule
				Parameter regime
				& Upper bound
				& Lower bound
				& Optimality \\
				\midrule
				
				\(\log(\delta^{-1})\asymp\log N,\)\newline
				\(\log C\asymp\log N\)
				& \(W^2L^2\log W\lesssim N\log N\)
				& \(W^2L^2\gtrsim N\)
				& Optimal\newline
				\mbox{if \(\log W\asymp\log N\)} \\

				\addlinespace[7pt]
				
				\(\log(\delta^{-1})\asymp\log N,\)\newline
				\(\log C=o(\log N)\)
				& \(W^2L^2\log W\lesssim N\log N\)
				& \(W^2L^2\gtrsim\frac{N\log C}{\log N}\)
				& Optimal up to\newline
				\(\cO(\log N/\log C)\)\newline
				\mbox{if \(\log W\asymp\log N\)} \\
				
				\addlinespace[7pt]
				
				\(\max\{\log(\delta^{-1}),\log C\}\)\newline
				\(\lesssim\log^k N\)
				& \(W^2L^2\log W\lesssim N\log^k N\)
				& \(W^2L^2\gtrsim\frac{N}{\log N}\)
				& Optimal up to\newline
				\(\cO(\log^k N)\)\newline
				\mbox{if \(\log W\asymp\log N\)} \\
				
				\midrule
				
				\(\max\{\log(\delta^{-1}),\log C\}\)\newline
				\(\gtrsim N\)
				& \(W^2L\lesssim N\)
				& \(W^2L\gtrsim N\)
				& Optimal \\
				
				\bottomrule
			\end{tabularx}
			
			\par\smallskip
			\footnotesize
			\noindent
			The first three rows summarize the adjustable width-depth
			bounds in Theorem~\ref{Theorem 2.1}, which hold for all admissible \(\gamma\).
			The upper bound in the last row follows from
			the interpolation construction using
			Lemma~\ref{Lemma 2.4}.
			All lower bounds follow from
			Theorem~\ref{Theorem 3.2}.
		\end{minipage}
	\end{table}

		\section{Conclusion}\label{Section 4}
		
		In this work, we investigate the expressive power of ReLU neural networks, focusing on the minimal network size required to memorize \(N\) labeled data points in the \(d\)-dimensional unit ball with pairwise separation distance at least \(\delta\). Unlike previous studies, we establish an explicit width-depth
		trade-off for memorizing these samples. For fixed \(d\), our constructions yield ReLU networks that memorize these samples with width \(W\) and depth
		\(L\) satisfying \(W^2L^2\log W\lesssim	N(\log(\delta^{-1})+\log C)\),
		where \(C\) is the number of possible labels. Besides, we also establish width-depth lower bounds that depend explicitly on \(\delta\) and \(C\). When \(\log(\delta^{-1})\asymp\log N\), these bounds show
		that our adjustable width-depth construction attains
		the optimal order \(W^2L^2\asymp N\) if
		\(\log C\asymp\log N\), and is optimal up to logarithmic
		factors if \(\log C=o(\log N)\). 

		Several directions for future research remain. First, our width-depth upper bounds are established by explicit network constructions rather than
		training algorithms. In particular, the binary encoding and bit-extraction procedures rely on high-precision representations of encoded data
		and may be sensitive to numerical errors. Moreover, our analysis does not establish whether standard training algorithms, such as gradient
		descent (GD) or stochastic gradient descent (SGD),	can find interpolating networks satisfying	comparable width-depth bounds. Whether such
		trade-offs can be achieved in practical training remains an open question.
		Second, it would be useful to study whether
		similar constructions and bounds hold for other activation
		functions, such as GELU and sigmoid. Third, one could investigate
		whether additional assumptions on the inputs, such as lying on
		a low-dimensional manifold or having few nonzero coordinates,
		allow memorization with smaller networks. Finally, closing the remaining logarithmic gaps between our upper and lower bounds remains an open problem.

		\section*{Acknowledgments}
		
		The work described in this paper was partially supported by National Natural Science Foundation of China under Grants 12501131 and 12526216.

		\bibliographystyle{myplainnat}
		\bibliography{References}
	\end{document}